\documentclass{article}

\usepackage{multirow}
\usepackage{graphicx} %
\usepackage{subfigure}
\usepackage[english]{babel}
\usepackage{natbib}

\usepackage{algorithm}
\usepackage{algorithmic}
\usepackage{amsthm}
\usepackage{latexsym, amsmath,amssymb, amsthm, mathtools}
\usepackage{amsmath, amsthm, amssymb}
\usepackage{hyperref}

\usepackage[accepted]{icml2013}
\newtheorem{theorem}{Theorem}

\newtheorem{lemma}{Lemma}
\icmltitlerunning{Mini-Batch Primal and Dual Methods for SVMs}

\newcommand{\setn}{\langle n \rangle}
\newcommand{\Exp}{\mathbb{E}}                      %

\usepackage{multicol}

\DeclareMathOperator{\Rand}{Rand}

\DeclareMathOperator*{\argmax}{arg\,max}

\newcommand{\st}{\;:\;}                          %
\newcommand{\ve}[2]{\left\langle #1 ,  #2 \right\rangle}    %
\newcommand{\eqdef}{:=}
\newcommand{\R}{\mathbb{R}}                      %
\newcommand{\E}{\mathbb{E}}                      %

\newcommand{\calG}{\mathbf{G}}

\newcommand{\calH}{\mathbf{H}}

\newcommand{\x}{ {\bf x}}
\newcommand{\vv}{ {\bf v}}

\newcommand{\X}{ {\bf X}}
\newcommand{\Q}{ {\bf Q}}

\newcommand{\alf}{ {\boldsymbol \alpha}}
\newcommand{\vchi}{ {\boldsymbol \chi}}

\newcommand{\vdelta}{{\boldsymbol \delta} }
\newcommand{\w}{  {\bf w}}

\newcommand{\vsubset}[2]{#1_{[#2]}}

\newcommand{\vc}[2]{#1^{(#2)}}                   %
\newcommand{\cor}[2]{{{#1}_{#2}}}                   %
\newcommand{\corit}[3]{{#1}_{#2}^{(#3)}}                   %

\newcommand{\norm}[1]{\left\lVert{#1}\right\rVert}
\newcommand{\hingeloss}{\ell}
\newcommand{\hinge}[1]{\hingeloss ( #1 )}
\newcommand{\trans}{{\top}}

\newcommand{\bP}{\mathbf{P}}
\newcommand{\bD}{\mathbf{D}}

\newcommand{\bH}{\mathbf{H}}

\begin{document}

\twocolumn[
\icmltitle{Mini-Batch Primal and Dual Methods for SVMs}
\icmlauthor{Martin Tak\'{a}\v{c}}{martin.taki@gmail.com}
\icmladdress{University of Edinburgh,
            JCMB, King's Buildings, EH9 3JZ, Edinburgh, UK}
\icmlauthor{Avleen Bijral}{abijral@ttic.edu}
\icmladdress{Toyota Technological Institute at Chicago,
            6045 S. Kenwood Ave., Chicago, Illinois 60637, USA}
\icmlauthor{Peter Richt\'{a}rik}{peter.richtarik@ed.ac.uk}
\icmladdress{University of Edinburgh,
            JCMB, King's Buildings, EH9 3JZ, Edinburgh, UK}
\icmlauthor{Nathan Srebro}{nati@ttic.edu}
\icmladdress{Toyota Technological Institute at Chicago,
            6045 S. Kenwood Ave., Chicago, Illinois 60637, USA}

\icmlkeywords{mini-batching, support vector machine, stochastic gradient descent, stochastic dual coordinate ascent, Pegasos, parallel coordinate descent, spectral norm, iteration complexity}

\vskip 0.3in
]

\begin{abstract}
  We address the issue of using mini-batches in stochastic
  optimization of SVMs. We show that the same quantity, the
  \emph{spectral norm of the data}, controls the parallelization
  speedup obtained for both primal stochastic subgradient descent
  (SGD) and stochastic dual coordinate ascent (SCDA) methods and use
  it to derive novel variants of mini-batched SDCA. Our guarantees for
  both methods are expressed in terms of the original nonsmooth primal
  problem based on the hinge-loss.
\end{abstract}

\section{Introduction}

Stochastic optimization approaches have been shown to have significant
theoretical and empirical advantages in training linear Support Vector
Machines (SVMs), as well as in many other learning applications, and
are often the methods of choice in practice.  Such methods use a
single, randomly chosen, training example at each iteration.  In the
context of SVMs, approaches of this form include primal stochastic
gradient descent (SGD) methods (e.g., Pegasos, \citealt{pegasos}, NORMA,
\citealt{zhang}) and dual stochastic coordinate ascent
\citep{dcd}.

However, the inherent sequential nature of such approaches becomes a
problematic limitation for parallel and distributed computations  as the predictor must be updated after each training point is
processed, providing very little opportunity for parallelization.  A
popular remedy is to use {\em mini-batches}.  That is, to use several
training points at each iteration, instead of just one, calculating
the update based on each point separately and aggregating the
updates.  The question is then whether basing each iteration on
several points can indeed reduce the number of required iterations,
and thus yield parallelization speedups.

In this paper, we consider using mini-batches with Pegasos (SGD on the
primal objective) and with Stochastic Dual Coordinate Ascent (SDCA). We show that for \emph{both methods}, the quantity that controls the speedup obtained using mini-batching/parallelization is the {\em spectral norm of the data}.

In Section~\ref{sec:sgd} we provide the first analysis of mini-batched Pegasos (with the
original, non-smooth, SVM objective) that provably leads to parallelization speedups (Theorem~\ref{thm:pegasos}). The idea of using mini-batches with Pegasos is not new, and is
discussed already by \citet{pegasos}, albeit without a theoretical
justification. The original Pegasos theoretical analysis does not
benefit from using mini-batches---the same number of iterations is
required even when large mini-batches are used, there is no
speedup, and the serial runtime (overall number of operations, in this
case data accesses) increases linearly with the mini-batch size.
In fact, no parallelization speedup can be guaranteed based only on a bound on
the radius of the data, as in the original Pegasos analysis.  Instead,
we provide a refined analysis based on the spectral norm of the data.

We then move on to SDCA (Section~\ref{sec:SDCA}). We show the situation is more involved, and
a modification to the method is necessary. SDCA has been consistently
shown to outperform Pegasos in practice \cite{dcd,pegasos}, and is
also popular as it does not rely on setting a step-size as in
Pegasos.  It is thus interesting and useful to obtain mini-batch
variants of SDCA as well. We first show that a naive
mini-batching approach for SDCA can fail, in particular when the
mini-batch size is large relative to the spectral norm (Section
\ref{sec:SDCA_naive}).  We then present a ``safe'' variant of
mini-batched SDCA, which depends on the spectral norm, and an analysis
for this safe variant that establishes the same
spectral-norm-dependent parallelization speedups as for Pegasos (Section \ref{sec:safeMiniBatching}).  Similar to
a recent analysis of non-mini-batched SDCA by \citet{ShalevShawartzZhang}, we establish a
guarantee on the duality gap, and thus also on the sub-optimality of
the {\em primal } SVM objective, when using mini-batched SDCA (Theorem
\ref{thm:dualityGapForLipFunctions}).  We then go on to describe a more
aggressive, adaptive, method for mini-batched SDCA, which is based on
the analysis of the ``safe'' approach, and which we show often
outperforms it in practice (Section~\ref{sec:aggressiveMiniBatching},
with experiments in Section~\ref{sec:experiments}).

For simplicity of presentation we focus on the hinge loss, as in the
SVM objective.  However, all our results for both Pegasos and SDCA are
valid for any Lipschitz continuous loss function.

\paragraph{Related Work.} Several recent papers consider the use of mini-batches in stochastic
gradient descent, as well as stochastic dual averaging and stochastic
mirror descent, when minimizing a {\em smooth} loss function
\citep{dekel2012optimal,AgarwalDuchi,CotterEtAl}.  These papers establish
parallelization speedups for {\em smooth} loss minimization with
mini-batches, possibly with the aid of some ``acceleration''
techniques, and without relying on, or considering, the spectral norm
of the data.  However, these results do not apply to SVM training, where
the objective to be minimized is the non-smooth hinge loss.  In fact, the
only data assumption in these papers is an assumption on the radius of
the data, which is {\em not} enough for obtaining parallelization
guarantees when the loss is non-smooth.  Our contribution is thus
orthogonal to these papers, showing that it is possible to obtain
parallelization speedups even for non-smooth objectives, but only with
a dependence on the spectral norm.  We also analyze SDCA, which is a
substantially different method from the methods analyzed in these
papers.  It is interesting to note that a bound of the spectral norm
could perhaps indicate that it is easier to ``smooth'' the objective,
and thus allow obtaining results similar to ours (i.e.~on the
suboptimality of the original non-smooth objective) by smoothing the
objective and relying on mini-batched smooth SGD, where the spectral
norm might control how well the smoothed loss captures the original loss.
But we are not aware of any analysis of this nature, nor whether such
an analysis is possible.

There has been some recent work on
mini-batched coordinate descent methods for $\ell_1$-regularized
problems (and, more generally, regularizes by a separable convex function), similar to the SVM dual.  \citet{shotgun} presented and
analyzed SHOTGUN, a parallel coordinate descent method  for
$\ell_1$-regularized problems, showing linear speedups for mini-batch
sizes bounded in terms of the spectral norm of the data.  The analysis
does not directly apply to the SVM dual because of the box
constraints, but is similar in spirit.  Furthermore,
\citet{shotgun} do not discuss a ``safe'' variant which is
applicable for any mini-batch size, and only study the analogue of what
we refer to as ``naive'' mini-batching (Section~\ref{sec:SDCA_naive}).
More directly related is recent work of \citet{richtarik,richtarikBigData}
which provided a theoretical framework and analysis for a more general
setting than SHOTGUN, that includes also the SVM dual as a special
case.  However, guarantees in this framework, as well as those of
\citet{shotgun}, are only on the dual suboptimality (in our
terminology), and not on the more relevant primal suboptimality,
i.e., the suboptimality of the original SVM problem we are interested
in.  Our theoretical analysis builds on that of
\citet{richtarikBigData}, combined with recent ideas of
\citet{ShalevShawartzZhang} for ``standard'' (serial) SDCA, to obtain bounds on
the duality gap and primal suboptimality.

\section{Support Vector Machines}

We consider the optimization problem of training a
linear\footnote{Since both Pegasos and SDCA can be kernelized, all
  methods discussed are implementable also with kernels, and all our
  results hold.  However, the main advantage of SGD and SDCA is where
  the feature map is given explicitly, and so we focus our
  presentation on this setting.}  Support Vector Machine (SVM) based
on $n$ labeled training examples $\{(\x_i,y_i)\}_{i=1}^n$, where
$\x_i\in\R^d$ and $y_i\in\pm 1$.  We use $\X = [\x_1,\dots,\x_n]
\in\R^{d \times n}$ to denote the matrix of training examples.  We
assume the data is normalized such that $\max_i \norm{\x_i} \leq 1$,
and thus suppress the dependence on $\max_i \norm{\x_i}$ in all
results.  Training a SVM corresponds to finding a linear predictor
$\w\in \R^d$ with low $\ell_2$-norm $\|\w\|$ and small (empirical)
average hinge loss $\hat{L}(\w) \eqdef \frac{1}{n} \sum_{i=1}^n
\hinge{y_i \ve{\w}{\x_i}}$, where $\hinge{z}\eqdef[1-z]_+ =
\max\{0,1-z\}$.  This bi-objective problem can be serialized as
\begin{equation}
  \label{eq:regSVM}
  \min_{\w\in\R^d} \left[\bP(\w) \eqdef \tfrac{1}{n} \sum_{i=1}^n  \hinge{y_i \ve{\w}{\x_i}}+ \tfrac{\lambda}{2} \norm{\w}^2\right],
\end{equation}
where $\lambda>0$ is a regularization trade-off parameter.  It is also
useful to consider the dual of \eqref{eq:regSVM}:
\begin{equation}
  \label{eq:dualSVM}
  \max_{\alf\in \R^n,0 \leq \alf_i \leq 1} \left[\bD(\alf) \eqdef
  \tfrac{-1}{2\lambda n^2} \alf^\trans \Q
  \alf+\tfrac{1}{n}\sum_{i=1}^n \alf_i\right],
\end{equation}
where
\begin{equation}\label{eq:Qii}\Q \in \R^{n \times n}, \quad \Q_{i,j} = y_i y_j \ve{\x_i}{\x_j},\end{equation}
is the Gram matrix of the (labeled) data.  The (primal) optimum of
\eqref{eq:regSVM} is given by $\w^*=\frac{1}{\lambda n} \sum_{i=1}^n
\alf^*_i y_i \x_i$, where $\alf^*$ is the (dual) optimum of
\eqref{eq:dualSVM}.  It is thus natural to associate with each dual
solution $\alf$ a primal solution (i.e., a linear predictor) \begin{equation}\label{eq:walpha}\w(\alf)
\eqdef \tfrac{1}{\lambda n} \sum_{i=1}^n \alf_i y_i \x_i.\end{equation}

We will be discussing ``mini-batches'' of size $b$, represented by \emph{random subsets} $A
\subseteq \setn \eqdef \{1,2,\dots,n\}$ of examples, drawn uniformly at random from all subsets of $\setn$ of cardinality $b$. Whenever we draw such a subset, we will for simplicity write $A\in \Rand(b)$. For $A\in \Rand(b)$ we use
$\Q_A \in \R^{b \times b}$ to denote the random submatrix of $\Q$ corresponding to rows
and columns indexed by $A$, $\vv_A \in \R^b$ to denote a similar
restriction of a vector $\vv \in \R^n$, and $\vv_{[A]}\in\R^n$ for the
``censored'' vector where entries inside $A$ are as in $\vv$ and entries
outside $A$ are zero.  The average hinge loss on examples in
$A$ is denoted by \begin{equation}\label{eq:9s9js9skk}
\hat{L}_A(\w) \eqdef \tfrac{1}{b} \sum_{i
  \in A}\hinge{y_i \ve{\w}{\x_i}}.\end{equation}

\section{Mini-Batches in Primal Stochastic Gradient Descent Methods}
\label{sec:sgd}

\begin{algorithm}[!htb]
   \caption{Pegasos with Mini-Batches}
   \label{alg:pegasos}
\begin{algorithmic}
   \STATE {\bfseries Input:} $\{(\x_i,y_i)\}_{i=1}^n$, $\lambda> 0$, $b\in \langle n \rangle$, $T \geq 1$
   \STATE {\bfseries Initialize:} set $\vc{\w}{1} = {\bf 0} \in \R^d$
   \FOR{$t=1$ {\bfseries to} $T$}
   \STATE Choose random mini-batch $A_t \in \Rand(b)$
   \STATE $\eta_t = \tfrac{1}{\lambda t}$,\; $A_t^{+} = \{i \in A_t \st y_i\langle \vc{\w}{t},\x_i\rangle <1 \}$
   \STATE $\vc{\w}{t+1} = (1-\eta_t \lambda)\vc{\w}{t} + \tfrac{\eta_{t}}{b} \sum_{i \in A_t^{+} } y_i\x_i$
   \ENDFOR
   \STATE {\bfseries Output:} $\bar{\w}^{(T)} = \tfrac{2}{T}\sum_{t=\lfloor T/2 \rfloor +1}^T \vc{\w}{t}$
\end{algorithmic}
\end{algorithm}

Pegasos is an SGD approach to solving
\eqref{eq:regSVM}, where at each iteration the iterate $\vc{\w}{t}$ is
updated based on an unbiased estimator of a sub-gradient of the
objective $\bP(\w)$.  Whereas in a ``pure'' stochastic setting, the
sub-gradient is estimated based on only a {\em single} training example,
in our mini-batched variation (Algorithm~\ref{alg:pegasos}) at each iteration we consider the partial
objective:
\begin{equation}\label{eq:ft}
\bP_t(\w) \eqdef \hat{L}_{A_t}(\w) + \tfrac{\lambda}{2}\norm{\w}^2,\\
\end{equation}
where $A_t\in \Rand(b)$. We then calculate the
subgradient of the partial objective $\bP_t$ at  $\vc{\w}{t}$:
\begin{equation}
  \label{eq:nablat}
      \boldsymbol{\nabla}^{(t)} \eqdef \nabla \bP_t(\vc{\w}{t}) \overset{\eqref{eq:ft}}{=} \nabla
      \hat{L}_{A_t}(\vc{\w}{t}) + \lambda \vc{\w}{t},
\end{equation}
where
\begin{equation}\label{eq:9sjs8}\nabla \hat{L}_{A}(\w) \overset{\eqref{eq:9s9js9skk}}{=} - \tfrac{1}{b} \sum_{i\in A} \chi_i(\w) y_i \x_i\end{equation} and
$\chi_i(\w) \eqdef 1$ if $y_i
\ve{\w}{\x_i} < 1$ and $0$ otherwise (indicator for not classifying example $i$ correctly with a margin).  The next iterate is obtained by setting
$\vc{\w}{t+1} =  \vc{\w}{t} - \eta_t \vc{\boldsymbol{\nabla}}{t}$. We can now write
\begin{equation}
  \label{eq:pegasosUpdate}
\vc{\w}{t+1} \overset{\eqref{eq:nablat}+\eqref{eq:9sjs8}}{=} (1-\eta_t \lambda) \vc{\w}{t} +  \tfrac{\eta_t}{b} \sum_{i\in A_t} \chi_i(\vc{\w}{t}) y_i \x_i.
\end{equation}

Analysis of mini-batched Pegasos rests on bounding the norm of the
subgradient estimates $\boldsymbol{\nabla}^{(t)}$.  An unconditional bound on this
norm, used in the standard Pegasos analysis, follows from bounding
\[\|\nabla \hat{L}_A(\w)\|   \overset{\eqref{eq:9sjs8}}{\leq} \tfrac{1}{b} \sum_{i\in A} \|\chi_i(\w) y_i \x_i\|
\leq \tfrac{1}{b} \sum_{i\in A} 1 = 1.\]
From \eqref{eq:nablat} we then get
$\|\boldsymbol{\nabla}^{(t)}\| \leq \lambda \|\vc{\w}{t}\| + 1$; the
standard Pegasos analysis follows.  This bound relies only on the assumption
$\max_i \|\x_i\|\leq 1$, and is the tightest  bound without further
assumptions on the data.

The core novel observation here is that the expected (square) norm of $\nabla \hat{L}_A$ can be bounded in terms of (an upper bound on)
{\em the spectral norm of the data:}
\begin{equation}\label{eq:sigma_sq}\sigma^2 \geq
\tfrac{1}{n}\norm{\X}^2  =  \tfrac{1}{n} \|\sum_i \x_i \x_i^\trans\| \overset{\eqref{eq:Qii}}{=}
\tfrac{1}{n}\norm{\Q},\end{equation}
where $\norm{\cdot}$ denotes the spectral norm (largest
singular value) of a matrix.   In order to bound $\nabla \hat{L}_A$, we
first perform the following calculation, introducing the key
quantity $\beta_b$, useful also in the analysis of SDCA.
\begin{lemma}\label{lemma:vQv}
  For any $\vv \in \R^n$, $\tilde{\Q} \in \R^{n \times n}$, $A \in \Rand(b)$,
  \begin{align*}
\E[ \vv^{\trans}_{[A]} \tilde{\Q} \vv_{[A]} ] &=
  \tfrac{b}{n}[ (1-\tfrac{b-1}{n-1})\sum_{i=1}^n \tilde{\Q}_{ii} \vv_i^2 + \tfrac{b-1}{n-1}
  \vv^{\trans} \tilde{\Q} \vv].\\
\intertext{Moreover, if $\tilde{\Q}_{ii} \leq 1$ for all $i$ and $\tfrac{1}{n}\|\tilde{\Q}\| \leq \sigma^2$, then}
\E[ \vv^{\trans}_{[A]} \tilde{\Q} \vv_{[A]} ] &\leq \tfrac{b}{n} \beta_b \norm{\vv}^2, \text{ where}
\end{align*}
\begin{equation}
  \label{eq:betab}
  \beta_b \eqdef 1+\tfrac{(b-1)(n \sigma^2 -1)}{n-1}.
\end{equation}
\end{lemma}

\begin{proof}
\begin{align*}
\lefteqn{\E[ \vv^{\trans}_{[A]} \tilde{\Q} \vv_{[A]} ] = \E[\sum_{i \in A}\vv_i^2 \tilde{\Q}_{ii} + \sum_{i,j\in A, i\ne j} \vv_i\vv_j \tilde{\Q}_{ij}]} \\
&\overset{(*)}{=} b \E_i [\vv_i^2\tilde{\Q}_{ii}] + b(b-1) \E_{i,j}[\vv_i\vv_j \tilde{\Q}_{ij}] \\
&= \tfrac{b}{n}\sum_i\tilde{\Q}_{ii}\vv_i^2 + \tfrac{b(b-1)}{n(n-1)}\vv^{\trans}(\tilde{\Q}-\text{diag}(\tilde{\Q}))\vv\\
&= \tfrac{b}{n}[(1-\tfrac{b-1}{n-1})\sum_i\tilde{\Q}_{ii}\vv_i^2 +
\tfrac{b-1}{n-1}\vv^{\trans}\tilde{\Q} \vv ],\\
\intertext{where in $(*)$ the expectations are over $i,j$ chosen uniformly at
random without replacement.  Now using $\tilde{\Q}_{ii} \leq 1$ and $\|\tilde{\Q}\|\leq n\sigma^2$, we can upper-bound the expectation as follows:}
&\leq \tfrac{b}{n}[(1-\tfrac{b-1}{n-1})\norm{\vv}^2 +
\tfrac{b-1}{n-1}n\sigma^2 \norm{\vv}^2] = \tfrac{b}{n}\beta_b
\norm{\vv}^2. \qedhere
\end{align*}
\end{proof}

We can now apply Lemma~\ref{lemma:vQv} to $\nabla \hat{L}_A$:
\begin{lemma}\label{lemma:nablaL}
  For any $\w\in \R^d$ and $A\in \Rand(b)$ we have
$\E [ \|\nabla \hat{L}_A(\w)\|^2 ] \leq
\frac{\beta_b}{b}$, where $\beta_b$ is as in Lemma~\ref{lemma:vQv}.
\end{lemma}
\begin{proof}
If $\chi\in\R^n$ is the vector with entries $\chi_i(\w)$, then
  \begin{multline*}
\E [ \| \nabla \hat{L}_A(\w) \|^2 ] \overset{\eqref{eq:9sjs8}}{=} \E [\| \tfrac{1}{b} \sum_{i\in A} \boldsymbol{\chi}_i y_i \x_i\|^2 ]     \\
\overset{\eqref{eq:Qii}}{=} \tfrac{1}{b^2} \E[\boldsymbol{\chi}_{[A]}^{\trans} \Q \boldsymbol{\chi}_{[A]}]
\overset{\textrm{(Lem\ref{lemma:vQv})}}{\leq} \tfrac{1}{b^2} \tfrac{b}{n}\beta_b
\norm{\boldsymbol{\chi}}^2 \leq \tfrac{\beta_b}{b}. \qedhere
\end{multline*}
\end{proof}

Using the by-now standard analysis of SGD for strongly convex
functions, we obtain the main result of this section:

\begin{theorem}\label{thm:pegasos}
  After $T$ iterations of Pegasos with mini-batches (Algorithm 1), we have that for the averaged iterate
  $\bar{\w}^{(T)} = \frac{2}{T} \sum_{t=\lfloor T/2 \rfloor+1}^T \vc{\w}{t}$:
$$\E\left[ \bP(\bar{\w}^{(T)}) \right] - \inf_{\w\in \R^d} \bP(\w)\leq
\tfrac{\beta_b}{b} \cdot \tfrac{30}{\lambda T}.$$
\end{theorem}

\begin{proof}
 Unrolling \eqref{eq:pegasosUpdate} with $\eta_t=1/(\lambda t)$ yields
  \begin{equation}\label{eq:hs8js8s}
  \vc{\w}{t} = -\tfrac{1}{\lambda (t-1)}\sum_{\tau=1}^{t-1} g^{(\tau)},
  \end{equation}
where $g^{(\tau)}\eqdef \nabla
  \hat{L}_{A_\tau}(\vc{\w}{\tau})$. Using  the inequality
$\|\sum_{\tau=1}^{t-1}  g^{(\tau)}\|^2  \leq (t-1) \sum_{\tau=1}^{t-1} \|g^{(\tau)}\|^2$, we now get
\begin{equation}\label{eq:js8sjs800}
\E [ \|\vc{\w}{t}\|^2 ] \overset{\eqref{eq:hs8js8s}}{\leq}
\sum_{\tau=1}^{t-1}\tfrac{\E [\|
g^{(\tau)}\|^2 ]}{\lambda^2(t-1)} \overset{\text{(Lem2)}}{\leq} \tfrac{\beta_b}{\lambda^2 b},
\end{equation}
\begin{equation*}
\E [\|\boldsymbol{\nabla}^{(t)}\|^2] \overset{\eqref{eq:nablat}+(\text{Lem}2)}{\leq} 2(\lambda^2 \E[ \|\vc{\w}{t}\|^2  ] + \tfrac{\beta_b}{b}) \overset{\eqref{eq:js8sjs800}}{\leq} 4\tfrac{\beta_b}{b}.
\end{equation*}
The performance guarantee is now given by the analysis of SGD with tail
averaging (Theorem~5 of \citealt{MakingSGDOptimal}, with $\alpha=\tfrac{1}{2}$ and $G^2=4\tfrac{\beta_b}{b}$).
\end{proof}

\textbf{Parallelization speedup.} When $b=1$ we have  $\beta_b=1$ (see \eqref{eq:betab}) and Theorem~\ref{thm:pegasos}
agrees with the standard (serial) Pegasos analysis\footnote{Except that we
  avoid the logarithmic factor by relying on tail averaging and a more
  modern SGD analysis.} \cite{pegasos}.  For larger
mini-batches, the guarantee depends on the quantity $\beta_b$, which
in turn depends on the spectral norm $\sigma^2$.  Since $\tfrac{1}{n} \leq
\sigma^2 \leq 1$, we have $1 \leq \beta_b \leq b$.

The worst-case situation is at a degenerate extreme, when all data points lie on
a single line, and so $\sigma^2=1$ and $\beta_b = b$.  In this case
Lemma~\ref{lemma:nablaL} degenerates to the worst-case bound of $\E[
\|\nabla \hat{L}_A(\mathbf{\w})\|^2] \leq 1$, and in Theorem~\ref{thm:pegasos} we have $\tfrac{\beta_b}{b}=1$, indicating that using larger
mini-batches does not help at all, and the same number of iteration
(i.e., the same parallel runtime, and $b$ times as much serial runtime)
is required.

However, when $\sigma^2<1$, and so $\beta_b < 1$, we see a benefit in
using mini-batches in Theorem~\ref{thm:pegasos}, corresponding to a
parallelization speedup of $\tfrac{b}{\beta_b}$.  The best situation is when $\sigma^2=\tfrac{1}{n}$,
and so $\beta_b=1$, which happens when all training points are
orthogonal.  In this case there is never any interaction between
points in the mini-batch, and using a mini-batch of size $b$ is just
as effective as making $b$ single-example steps.  When $\beta_b=1$ we
indeed see that the speedup speedup is equal to the number of mini-batches, and that the behavior in terms
of the number of data accesses (equivalently, serial runtime) $bT$,
does not depend on $b$; that is, even with larger mini-batches, we
require no more data accesses, and we gain linearly from being able to
perform the accesses in parallel.  The case $\sigma^2=\tfrac{1}{n}$ is rather
extreme, but even for intermediate values $\tfrac{1}{n} < \sigma^2 < 1$
we get speedup.  In particular, as long as $b \leq
\tfrac{1}{\sigma^2}$, we have $\beta_b \leq 2$, and an essentially linear
speedup.  Roughly speaking, $\tfrac{1}{\sigma^2}$ captures the number of
examples in the mini-batch beyond which we start getting significant
interactions between points.

\section{Mini-Batches in Dual Stochastic Coordinate Ascent Methods}\label{sec:SDCA}

An alternative stochastic method to Pegasos is Stochastic Dual
Coordinate Ascent (SDCA, \citealt{dcd}), aimed to solve the  dual
problem \eqref{eq:dualSVM}. At each iteration we choose a single
training example $(\x_i,y_i)$, uniformly at random, corresponding to a single
dual variable (coordinate) $\alf_i = e_i^\trans \alf$. Subsequently, $\alf_i$ is updated so
as to maximize the (dual) objective,  keeping all other
coordinates of $\alf$ unchanged and  maintaining the box
constraints.  At iteration $t$, the update $\vdelta_i^{(t)}$ to $\alf_i^{(t)}$ is computed via
\begin{eqnarray}
  \vdelta^{(t)}_i & \eqdef & \argmax_{0 \leq \alf^{(t)}_i + \delta \leq 1} \bD(\alf^{(t)}+ \delta e_i) \nonumber \\
&\overset{\eqref{eq:dualSVM}}{=}& \argmax_{0 \leq \alf^{(t)}_i + \delta \leq 1} (\lambda n- (\Q e_i)^\trans \alf^{(t)}) \delta - \tfrac{\Q_{i,i}}{2} \delta^2 \nonumber\\
&= & \textrm{clip}_{[-\alf^{(t)}_i,1-\alf^{(t)}_i]}\tfrac{\lambda n - (\Q e_i)^\trans \alf^{(t)}}{\Q_{i,i}} \nonumber \\
&\overset{\eqref{eq:Qii},\eqref{eq:walpha}}{=}& \textrm{clip}_{[-\alf^{(t)}_i,1-\alf^{(t)}_i]}\tfrac{\lambda
  n(1-y_i \ve{\w(\alf^{(t)})}{\x_i})}{\|\x_i\|^2}, \label{eq:alphai}
\end{eqnarray}
where $\textrm{clip}_{I}$ is  projection onto the interval $I$. Variables $\alf_j^{(t)}$ for $j\neq i$ are unchanged.
 Hence, a single iteration has the form
$\alf^{(t+1)} = \alf^{(t)} + \vdelta_i^{(t)}e_i$.
Similar to a Pegasos update, at each iteration a single, random,
training point is considered, the ``response'' $y_i
\ve{\w(\alf^{(t)})}{\x_i}$ is calculated (this operation dominates the
computational effort), and based on the response, a multiple of $\x_i$
is added to the weight vector $\w$ (corresponding to changing
$\alf_i$).  The two methods thus involve fairly similar operations
at each iteration, with essentially identical computational costs.
They differ in that in Pegasos, $\alf_i$ is changed according to some
pre-determined step-size, while SDCA changes it optimally so as to
maximize the dual objective (and maintain dual feasibility); there
is no step-size parameter.

SDCA was suggested and studied empirically by \citet{dcd}, where
empirical advantages over Pegasos were often observed.  In terms of a
theoretical analysis, by considering the dual problem
\eqref{eq:dualSVM} as an $\ell_1$-regularized, box-constrained
quadratic problem, it is possible to obtain guarantees on the {\em
  dual} suboptimality, $\bD(\alf^*) - \bD(\alf^{(t)})$, after a
finite number of SDCA iterations \cite{TewariShalevShwartzJMLR2011,NesterovRCDM,richtarik}. However, such
guarantees do {\em not} directly imply guarantees on the {\em primal}
suboptimality of $\w(\alf^{(t)})$.  Recently,
\citet{ShalevShawartzZhang} bridged this gap, and provided guarantees
on $\bP(\w(\alf^{(t)}))-\bP(\w^*)$ after a finite number of SDCA
iterations.  These guarantees serve as the starting point for our
theoretical study.

\subsection{Naive Mini-Batching}
\label{sec:SDCA_naive}
A naive approach to parallelizing SDCA using mini-batches is to
compute $\vdelta_i^{(t)}$ in parallel, according to \eqref{eq:alphai},
for all $i \in A_t$, all based on the current iterate $\alf^{(t)}$,
and then update $\alf^{(t+1)}_i=\alf^{(t)}_i+ \vdelta^{(t)}_i$ for
$i \in A_t$, and keep $\alf^{(t+1)}_j=\alf^{(t)}_j$ for $j\not\in A_t$.
However, not only might this approach not reduce the number of
required iterations, it might actually {\em increase} the number of
required iterations. This is because the dual objective need  {\em not improve
  monotonically} (as it does for ``pure'' SDCA), and even not
converge.

To see this, consider an extreme situation with only two identical
training examples: $\Q =
\begin{bsmallmatrix}1&1\\1&1\end{bsmallmatrix}$, $\lambda=\tfrac{1}{n}=\tfrac{1}{2}$ and
mini-batch size $b=2$ (i.e., in each iteration we use both
examples). If we start with $\vc{\alf}{0}={\bf 0}$
with $\bD(\vc{\alf}{0})=0$
then $\vc{\vdelta_1}{0}=\vc{\vdelta_2}{0}=1$ and
following the naive approach we have $\vc{\alf}{1}=(1,1)^T$
with objective value $\bD(\vc{\alf}{1})=0$.
In the next iteration $\vc{\vdelta_1}{1}=\vc{\vdelta_2}{1}=-1$
which brings us back to $\vc{\alf}{2}={\bf 0}$.
So the algorithm will alternate between those two solutions
with objective value $\bD(\alf)=0$, while at the optimum
$\bD(\alf^*)=\bD((0.5,0.5)^\trans)=0.25$.

This is of course a simplistic toy example, but the same phenomenon
will occur when a large number of training examples are identical or
highly correlated.  This can  also be observed empirically in some
of our experiments discussed later, e.g., in Figure~\ref{fig:failureOfNaiveApproach}.

The problem here is that since we update each $\alf_i$ independently to
its optimal value {\em as if all other coordinates were fixed}, we are
ignoring interactions between the updates.  As we see in the
extreme example above, two different $i,j\in A_t$, might suggest
essentially the same change to $\w(\alf^{(t)})$, but we would then perform
this update {\em twice}, overshooting and yielding a new iterate
which is actually worse then the previous one.

\subsection{Safe Mini-Batching}
\label{sec:safeMiniBatching}

Properly accounting for the interactions between coordinates in the
mini-batch would require jointly optimizing over all $\alf_i$, $i\in
A_t$.  This would be a very powerful update and no-doubt reduce the
number of required iterations, but would require solving a
box-constrained quadratic program, with a quadratic term of the form
$\vdelta_A^{\trans} \Q_A \vdelta_A$, $\vdelta_A \in \R^b$, at each iteration.  This
quadratic program cannot be distributed to different machines, each
handling only a single data point.

Instead, we propose a ``safe'' variant, where the term $\vdelta_A^{\trans} \Q_A
\vdelta_A$ is approximately bounded by the separable surrogate
$\beta \norm{\vdelta_A}^2$, for some $\beta>0$ which we will discuss
later.  That is, the update is given by:
\begin{eqnarray}
  \vdelta^{(t)}_i &\eqdef & \argmax_{0 \leq \alf^{(t)}_i + \delta \leq
    1} (\lambda n - (\Q e_i)^\trans \alf^{(t)}) \delta - \tfrac{\beta}{2}
  \delta^2  \nonumber \\
&=& \textrm{clip}_{[-\alf^{(t)}_i,1-\alf^{(t)}_i]}\tfrac{\lambda
  n(1-y_i \ve{w(\alf^{(t)})}{\x_i})}{\beta},  \label{eq:MBalphai}
\end{eqnarray}
with $\alf^{(t+1)}_i=\alf^{(t)}_i+ \vdelta^{(t)}_i$ for $i \in
A_t$, and $\alf^{(t+1)}_j=\alf^{(t)}_j$ for $j\not\in A_t$.  In
essence, $\tfrac{1}{\beta}$ serves as a step-size, where we are now careful not
to take steps so big that they will accumulate together and overshoot
the objective.  If handling only a single point at each iteration,
such a short-step approach is not necessary, we do not need a
step-size, and we can take a ``full step'', setting $\alf_i$
optimally ($\beta=1$).  But with the potential for interaction between
coordinates updated in parallel, we must use a smaller step, depending
on the potential for such interactions.

We will first rely on the bound \eqref{eq:sigma_sq}, and establish that
the choice $\beta=\beta_b$ as in \eqref{eq:betab} provides for a safe
step size. To do so, we consider the dual objective at $\alf+\vdelta$,
\begin{equation}\label{eq:Ddelta}
  \bD(\alf+\vdelta) = -\tfrac{(\alf^\trans \Q
  \alf + 2 \alf^\trans \Q \vdelta + \vdelta^\trans \Q \vdelta)}{2\lambda n^2}  +\sum_{i=1}^n  \tfrac{\alf_i+\vdelta_i}{n},
\end{equation}
and the following separable approximation to it:
\begin{equation}  \label{eq:H}
  \calH(\vdelta,\alf) \eqdef -\tfrac{(\alf^\trans \Q
  \alf + 2 \alf^\trans \Q \vdelta + \beta_b \norm{\vdelta}^2)}{2\lambda n^2}  + \sum_{i=1}^n \tfrac{\alf_i+\vdelta_i}{n},
\end{equation}
in which $\beta_b \norm{\vdelta}^2$ replaces $\vdelta^\trans \Q \vdelta$.
Our update \eqref{eq:MBalphai} with $\beta=\beta_b$ can be written as
$\displaystyle\vdelta = \arg\max_{\vdelta: 0 \leq \alf+\vdelta \leq 1} \calH(\vdelta,\alf)$
(we then use the coordinates $\vdelta_i$ for $i\in A$ and ignore the
rest). We are essentially performing parallel coordinate ascent on the
separable approximation $\calH(\vdelta,\alf)$ instead of on
$\bD(\alf+\vdelta)$.  To understand this approximation, we note that
$\bH({\bf 0},\alf)=\bD(\alf)$, and show that $\calH(\vdelta,\alf)$ provides an
approximate expected lower bound on $\bD(\alf+\vdelta)$:
\begin{lemma}\label{lemma:H}
  For any $\alf,\vdelta\in \R^n$ and $A \in \Rand(b)$,
$$\E_A[ \bD(\alf+\vdelta_{[A]}) ] \geq (1-\tfrac{b}{n})\bD(\alf) + \tfrac{b}{n}\calH(\vdelta,\alf).$$
\end{lemma}
\begin{proof}
  Examining \eqref{eq:Ddelta} and \eqref{eq:H}, the terms that do not
  depend on $\vdelta$ are equal on both sides.  For the linear term in
  $\vdelta$, we have that $\E[\vdelta_{[A]}] = \frac{b}{n} \vdelta$, and
  again we have equality on both sides.  For the quadratic term we use
  Lemma \ref{lemma:vQv} which yields $\E[ \vdelta_{[A]}^\trans \Q
  \vdelta_{[A]} ]\leq \frac{b}{n}\beta_b\norm{\vdelta}^2$, and after
  negation establishes the desired bound.
\end{proof}

Inequalities of this general type are also studied in
\cite{richtarikBigData} (see Sections~3 and 4).  Based on
the above lemma, we can modify the analysis of
\citet{ShalevShawartzZhang} to obtain (see complete proof in the
appendix):

\begin{theorem}
\label{thm:dualityGapForLipFunctions}
Consider the SDCA updates given by \eqref{eq:MBalphai}, with $A_t \in \Rand(b)$, starting from $\alf^{(0)}={\bf 0}$ and with
$\beta = \beta_b$ (given in eq.~\eqref{eq:betab}).
For any $\epsilon>0$ and
\begin{eqnarray}\label{eq:dualityRequirements}
t_0 &\geq&   \max\{ 0,\lceil \tfrac nb \log(\tfrac{2\lambda n}{\beta_b}) \rceil\},
  \\
T_0 &\geq & t_0+
\tfrac{\beta_b}b
\left[\tfrac{4}{ \lambda \epsilon} -2\tfrac n{\beta_b} \right]_+
,\\
T &\geq & T_0 + \max\{\lceil \tfrac nb\rceil,\tfrac{\beta_b}{b}
\tfrac{ 1}{\lambda \epsilon}\},\\
 \bar{\alf} &\eqdef& \tfrac{1}{T-T_0} \sum_{t=T_0}^{T-1} \alf^{(t)},\label{eq:averageDefinition}
\end{eqnarray}
we have $$ \E[ \bP(\w(\bar{\alf})) ] - \bP(\w^*) \leq \E[ \bP(\w(\bar{\alf})) -
\bD(\bar{\alf}) ] \leq \epsilon.$$
\end{theorem}

The number of iterations of mini-batched SDCA, sufficient to
reach \emph{primal} suboptimality $\epsilon$, is by Theorem~\ref{thm:dualityGapForLipFunctions} equal to
\begin{equation}
  \label{eq:sdcaspeedup}
\tilde{O}\left( \tfrac{n}{b} + \tfrac{\beta_b}{b} \cdot \tfrac{1}{\lambda
  \epsilon}\right).
\end{equation}
We observe the \emph{same speedup} as in the case of mini-batched Pegasos: factor of
$\tfrac{b}{\beta_b}$, with an essentially linear speedup when $b \leq
\tfrac{1}{\sigma^2}$.  It is interesting to note that the quantity $\beta_b$
only affects the second, $\epsilon$-dependent, term in
\eqref{eq:sdcaspeedup}.  The ``fixed cost'' term, which essentially
requires a full pass over the data, is {\em not} affected by
$\beta_b$, and is {\em always} scaled down by $b$.

\subsection{Aggressive Mini-Batching}
\label{sec:aggressiveMiniBatching}

Using  $\beta = \beta_\sigma$ is safe, but might be too
safe/conservative.
In particular, we used the spectral norm to bound
$\vdelta^\trans \Q \vdelta \leq \norm{\Q}\norm{\vdelta}^2$ in
Lemma~\ref{lemma:H} (through Lemma \ref{lemma:vQv}), but this is a worst
case bound over all possible vectors, and might be loose for the
relevant vectors $\vdelta$.  Relying on a worst-case bound might mean
we are taking much smaller steps then we could be.
Furthermore, the approach we presented thus far relies on knowing the
spectral norm of the data, or at least a bound on the spectral norm
(recall \eqref{eq:sigma_sq}), in order to set the step-size.  Although
it is possible to estimate this quantity by sampling, this can
certainly be inconvenient.

Instead, we suggest a more aggressive variant of mini-batched SDCA
which gradually adapts $\beta$ based on the actual values of
$\|\vdelta_{[A_t]}^{(t)}\|^2$ and $\vdelta^{(t)}_{[A_t]} \Q
\vdelta^{(t)}_{[A_t]}$.
\begin{algorithm}[]
   \caption{SDCA with Mini-Batches (aggressive)}
   \label{alg:aggressive}
\begin{algorithmic}
   \STATE {\bfseries Input:} $\{(\x_i,y_i)\}_{i=1}^n$, $\lambda > 0$, $b\in \R^d$, $T\geq 1$, $\gamma = 0.95$
   \STATE {\bfseries Initialize:} set $\vc{\alf}{0} = {\bf 0}$,
   $\vc{\w}{0}={\bf 0}$,
   $\vc{\beta}{0} = \beta_b$

   \FOR{$t=0$ {\bfseries to} $T$}
   \STATE Choose $A_t \in \Rand(b)$
   \STATE For $i\in A_t$, compute $\cor{\tilde{\vdelta}}{i}$ from
   \eqref{eq:MBalphai} using $\beta=\vc{\beta}{t}$
   \STATE Sum $\zeta :=\sum_{i \in A_t}
   \tilde\vdelta^2_i$ and $\tilde{\Delta}:=\sum_{i \in A_t} \tilde{\vdelta}_i y_i \x_i$
   \STATE Compute $\rho=
   \textrm{clip}_{[1,\beta_b]}\left(\frac{\norm{\tilde{\Delta}}^2}{\zeta}\right)$
   \STATE For $i\in A_t$, compute $\cor{\vdelta}{i}$ from
   \eqref{eq:MBalphai} using $\beta=\rho$.
   \STATE $\vc{\beta}{t+1} :=
     (\vc{\beta}{t})^{\gamma} \rho^{1-\gamma}$
   \IF{$\bD(\vc{\alf}{t}+{\vdelta}_{[A_t]})>\bD(\vc{\alf}{t})$}
     \STATE $\vc{\alf}{t+1}=\vc{\alf}{t}+\vdelta_{[A_t]}$,
     \STATE $\vc{\w}{t+1}=\vc{\w}{t}+\frac{1}{\lambda n}\sum_{i \in A_t} \vdelta_i y_i \x_i$
   \ELSE
     \STATE $\vc{\alf}{t+1}=\vc{\alf}{t}$, $\vc{\w}{t+1}=\vc{\w}{t}$
   \ENDIF
   \ENDFOR
\end{algorithmic}
\end{algorithm}
In Section~\ref{sec:experiments} one can observe advantages of this
aggressive strategy.

In this variant, at each iteration we calculate the ratio
$\tilde{\vdelta}_{[A]}^\trans\Q   \tilde{\vdelta}_{[A]}^\trans  / \|\tilde{\vdelta}_{[A]}\|^2 $, and nudge the step
size towards it by updating it to a weighted geometric average of the
previous step size and ``optimal'' step size based on the step $\vdelta$
considered.  One complication is that due to the box constraints, not
only the magnitude but also the direction of the step $\vdelta$
depends on the step-size $\beta$, leading to a circular situation.
The approach we take is as follows: we maintain a ``current step
size'' $\beta$.  At each iteration, we first calculate a tentative
step $\tilde{\vdelta}_A$, according to \eqref{eq:MBalphai}, with the
current $\beta$.  We then calculate $\rho =
\frac{\tilde{\vdelta}_{[A]}^\trans\Q
  \tilde{\vdelta}_{[A]}^\trans}{\norm{\tilde{\vdelta}_{[A]}}^2}$,
according to this step direction, and update $\beta$ to
$\beta^{\gamma}\rho^{1-\gamma}$ for some pre-determined parameter
$0<\gamma<1$ that controls how quickly the step-size adapts.  But,
instead of using $\tilde{\vdelta}$ calculated with the previous
$\beta$, we actually re-compute $\vdelta_A$ using the step-size
$\rho$.  We note that this means the ratio $\rho$ does {\em not}
correspond to the step $\vdelta_A$ actually taken, but rather to the
tentative step $\tilde{\vdelta}_A$.  We could potentially continue
iteratively updating $\rho$ according to $\vdelta_A$ and $\vdelta_A$
according to $\rho$, but we found that this does not improve
performance significantly and is generally not worth the extra
computational effort.  This aggressive strategy is summarized in
Algorithm \ref{alg:aggressive}.  Note that we initialize
$\beta=\beta_b$, and also constrain $\beta$ to remain in the range
$[1,\beta_b]$, but we can use a very crude upper bound $\sigma^2$ for
calculating $\beta_b$.  Also, in our aggressive strategy, we refuse
steps that do not actually increase the dual objective,
corresponding to overly aggressive step sizes.

Carrying out the aggressive strategy requires computing
$\tilde{\vdelta}_{[A]}^\trans\Q \tilde{\vdelta}_{[A]}$ and the dual
objective efficiently and in parallel.  The main observation here is
that:
  \begin{equation}
    \label{eq:deltaQdelta}
    \tilde{\vdelta}_{[A]}^\trans\Q \tilde{\vdelta}_{[A]} = \norm{
      \sum_{i\in A} \tilde{\vdelta}_i y_i \mathbf{\x_i}}^2
  \end{equation}
  and so the main operation to be performed is an aggregation of
  $\sum_{i\in A} \tilde{\vdelta}_i y_i \x_i$, similar to the operation
  required in mini-batched Pegasos.  As for the dual objective, it can
  be written as $\bD(\alf) = -\norm{\w(\alf)}^2 -
  \frac{1}{n}\norm{\alf}_1$ and can thus be readily calculated if we
  maintain $\w(\alf)$, its norm, and
  $\norm{\alf}_1$.

\section{Experiments} \label{sec:experiments}

\begin{figure*}[ht!]
\centering
\subfigure{\includegraphics[width=1.5in]{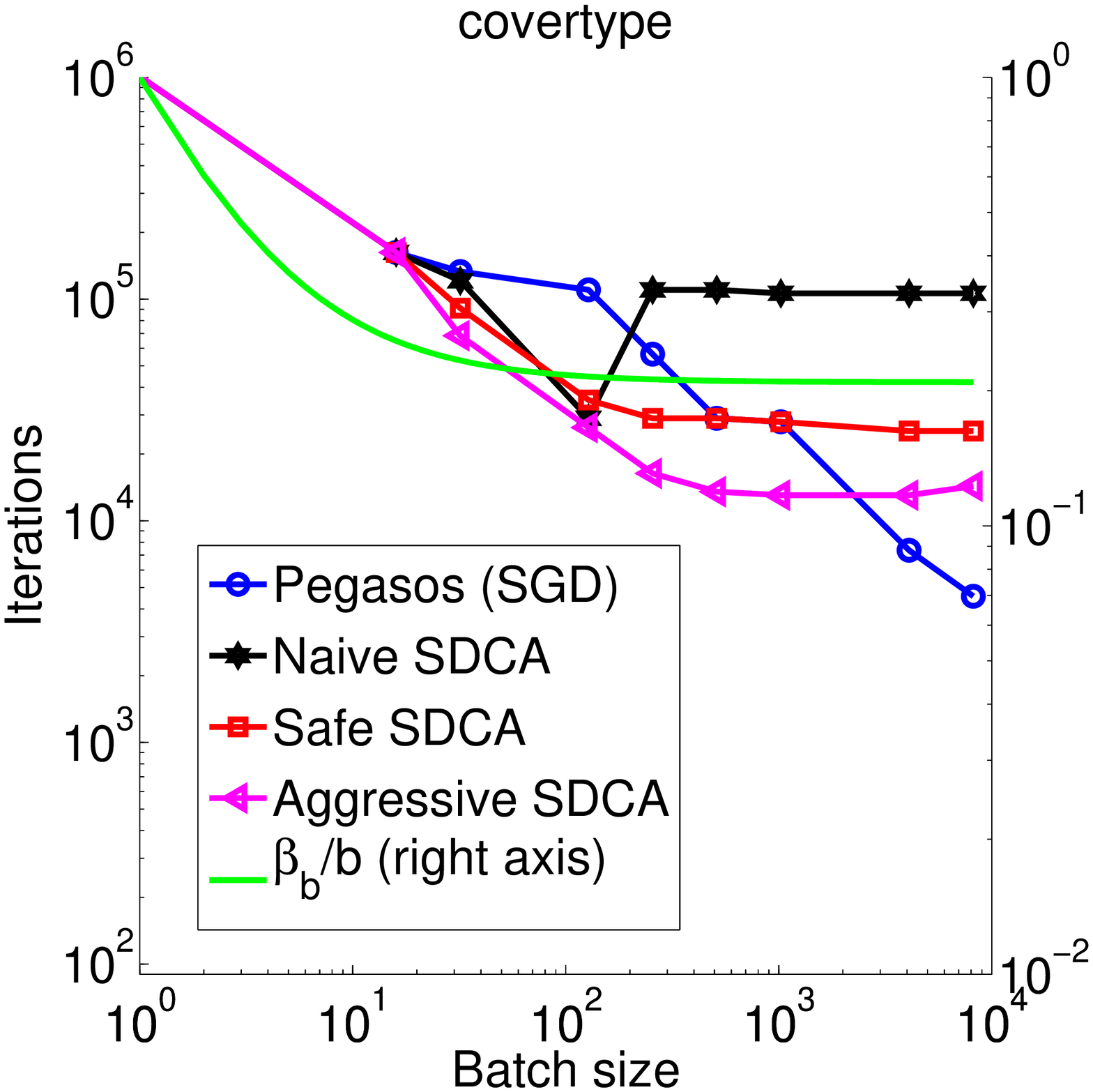}}
\subfigure{\includegraphics[width=1.5in]{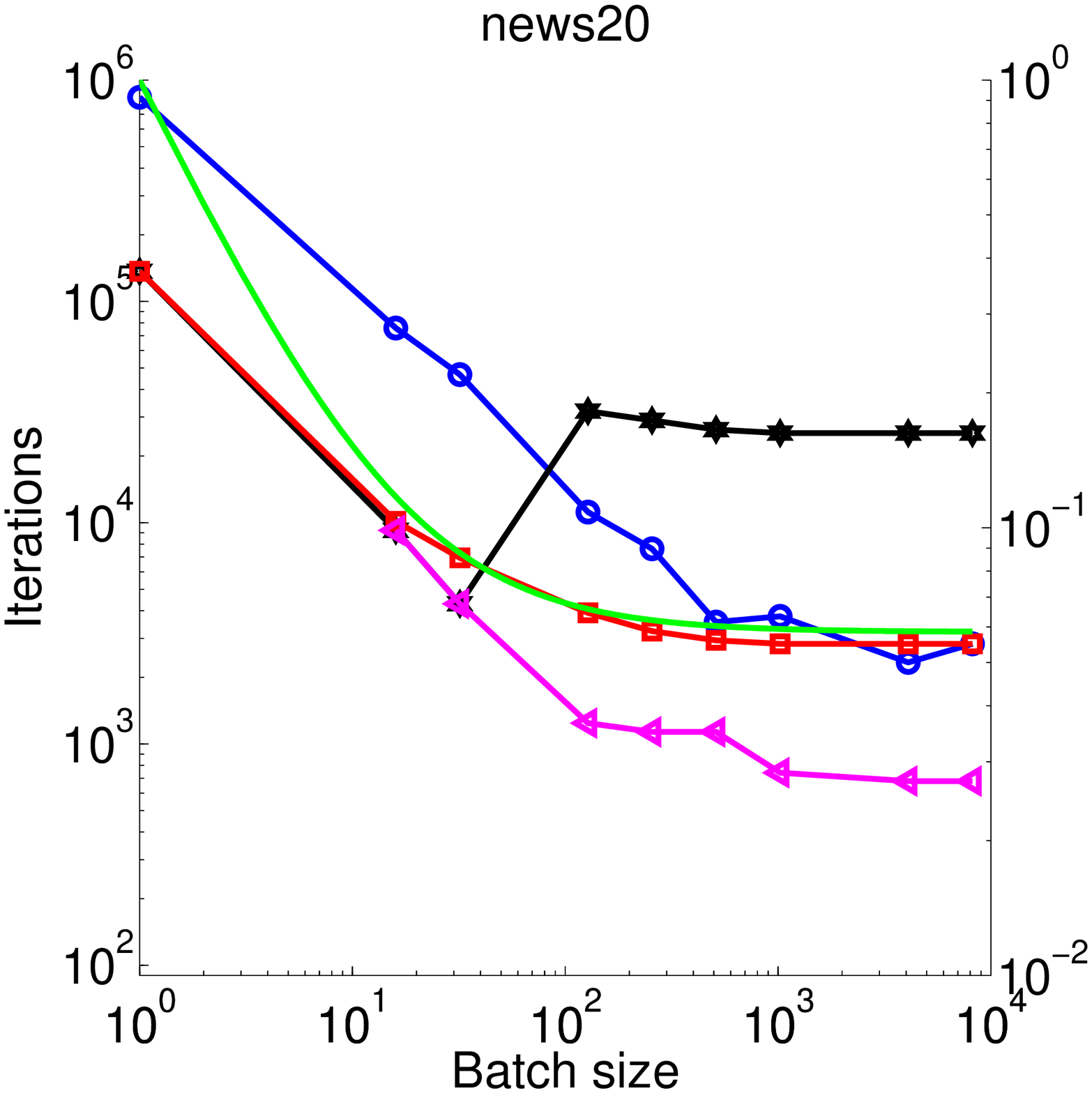}}
\subfigure{\includegraphics[width=1.5in]{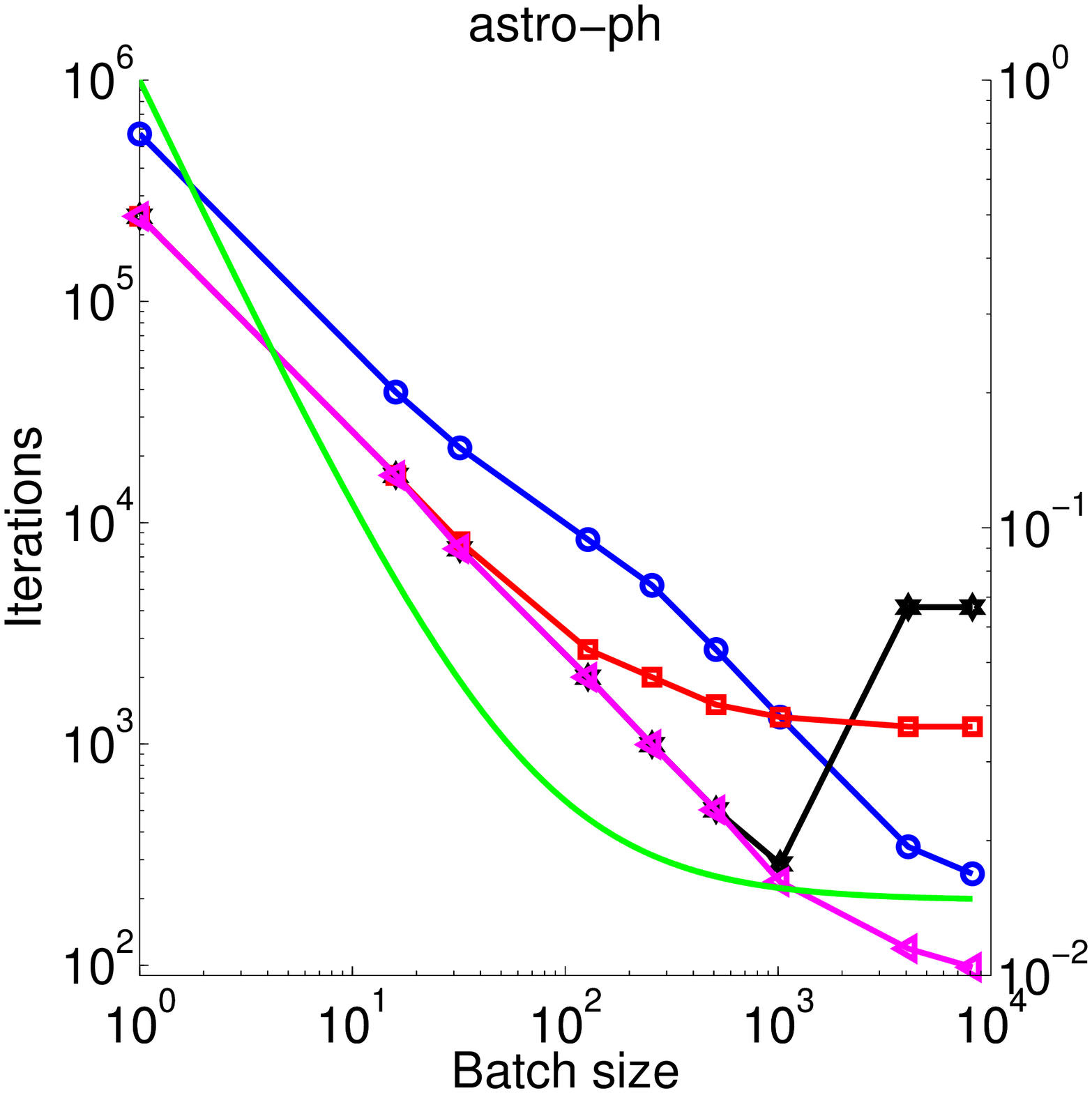}}
\subfigure{\includegraphics[width=1.5in]{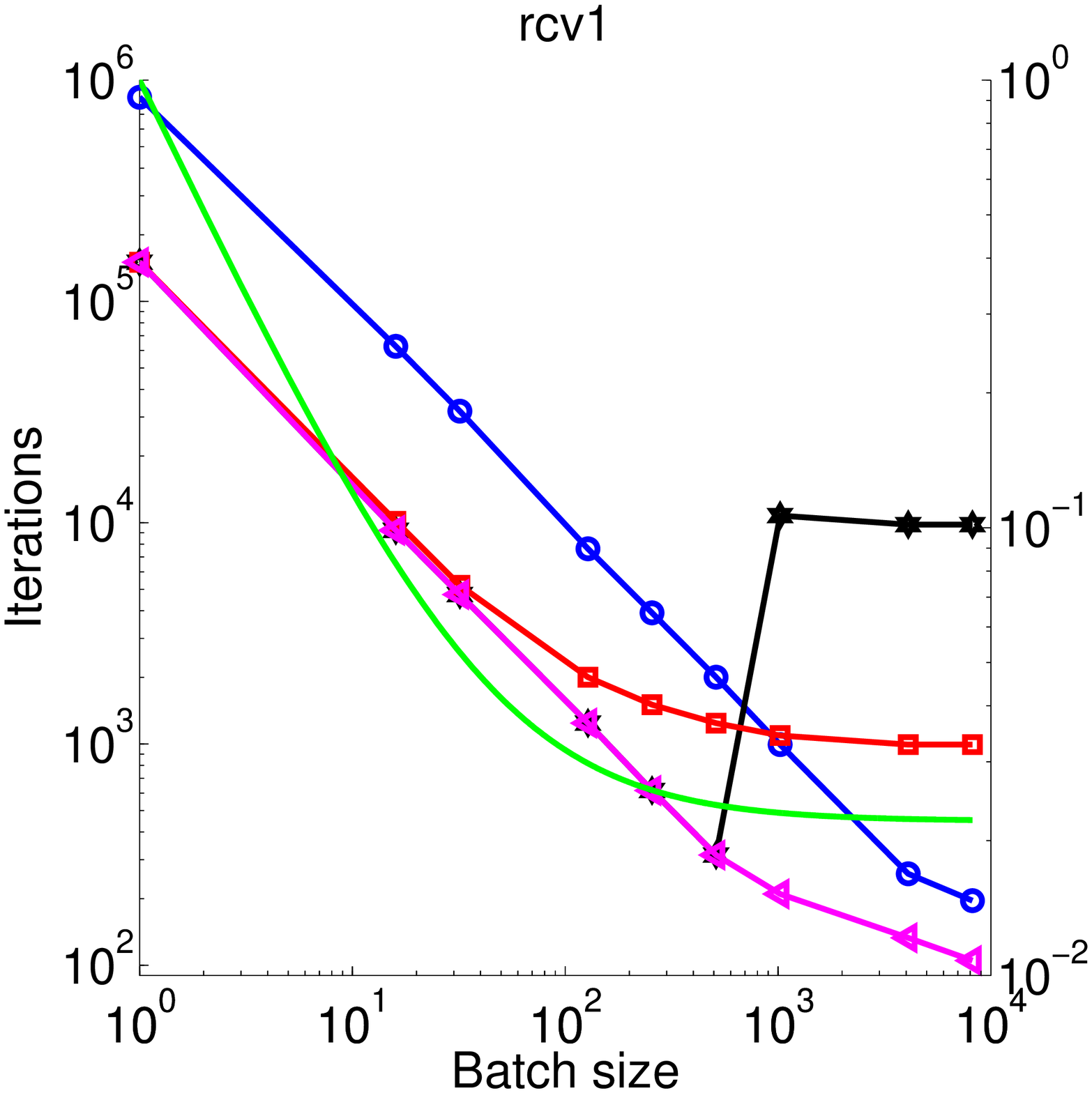}}
\vspace{-0.2in}
\small
 \caption{\small Number of iterations (left vertical axis) needed to find a $0.001$-accurate \emph{primal} solution for different mini-batch sizes $b$ (horizontal axis). The leading factor in our analysis, $\beta_b/b$, is plotted on the right vertical axis.}
 \label{fig:itertions2epsilon}
\end{figure*}

\begin{figure*}[ht!]
\centering
\subfigure{\includegraphics[width=1.5in]{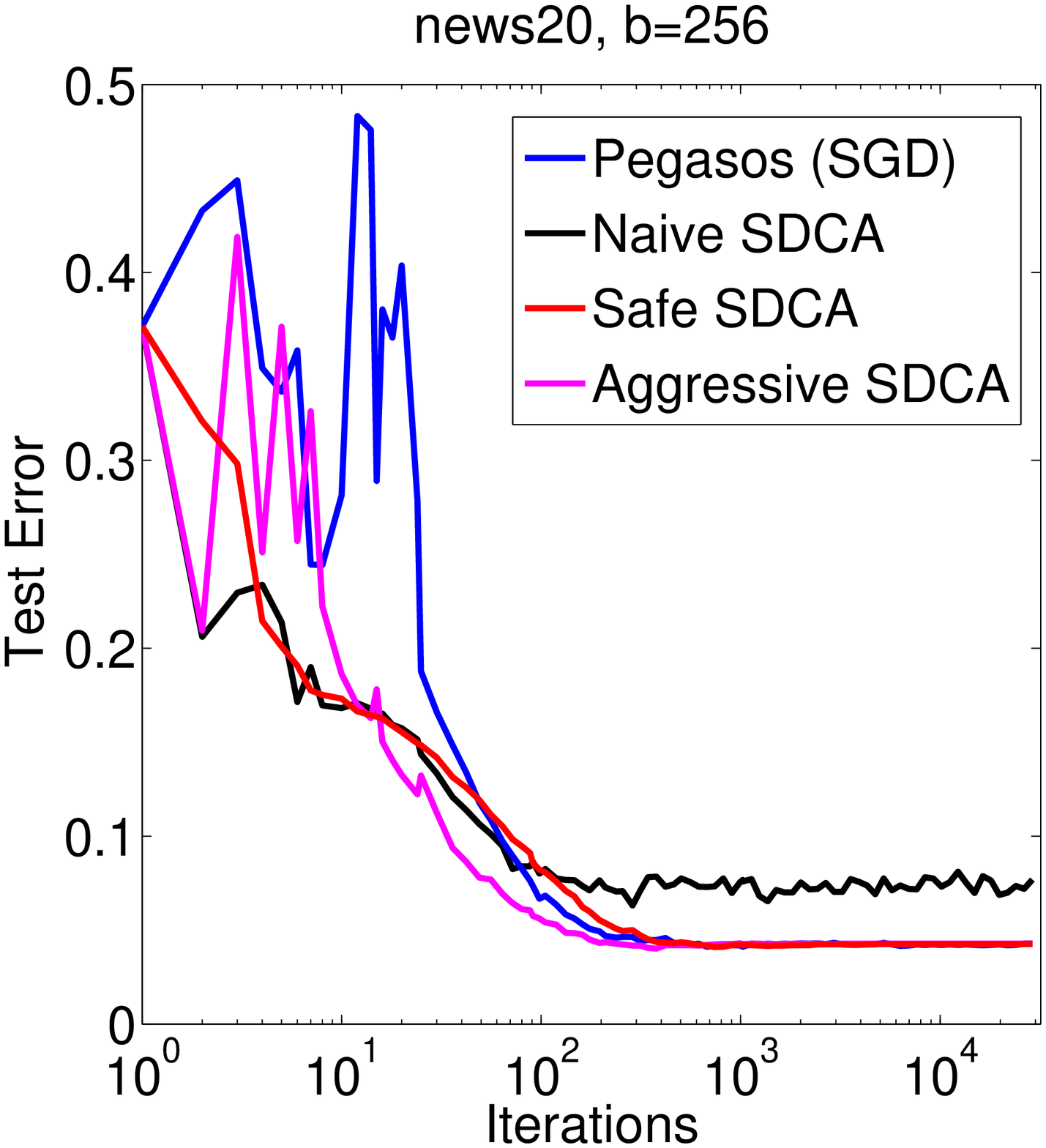}}
\subfigure{\includegraphics[width=1.5in]{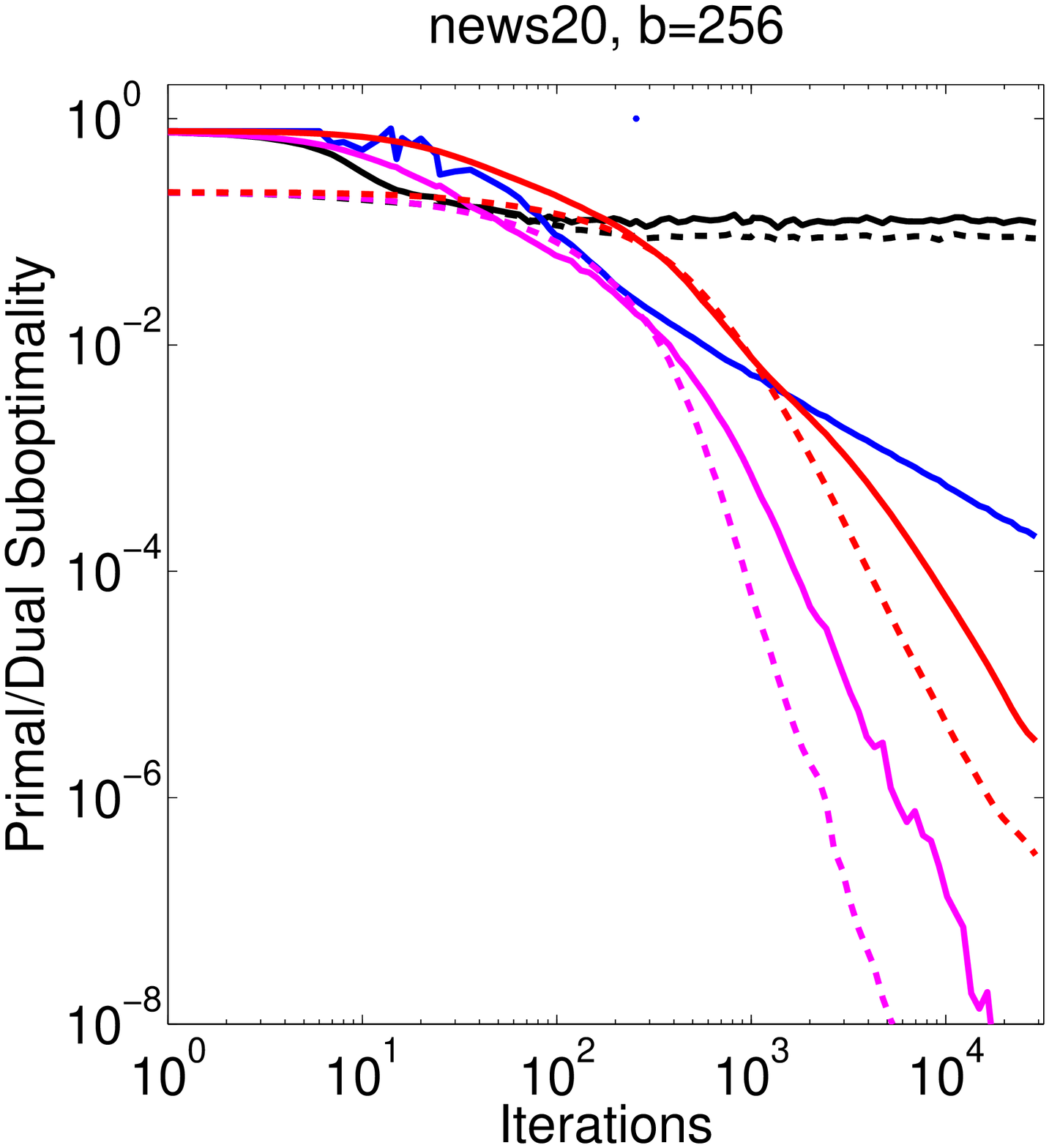}}
\subfigure{\includegraphics[width=1.5in]{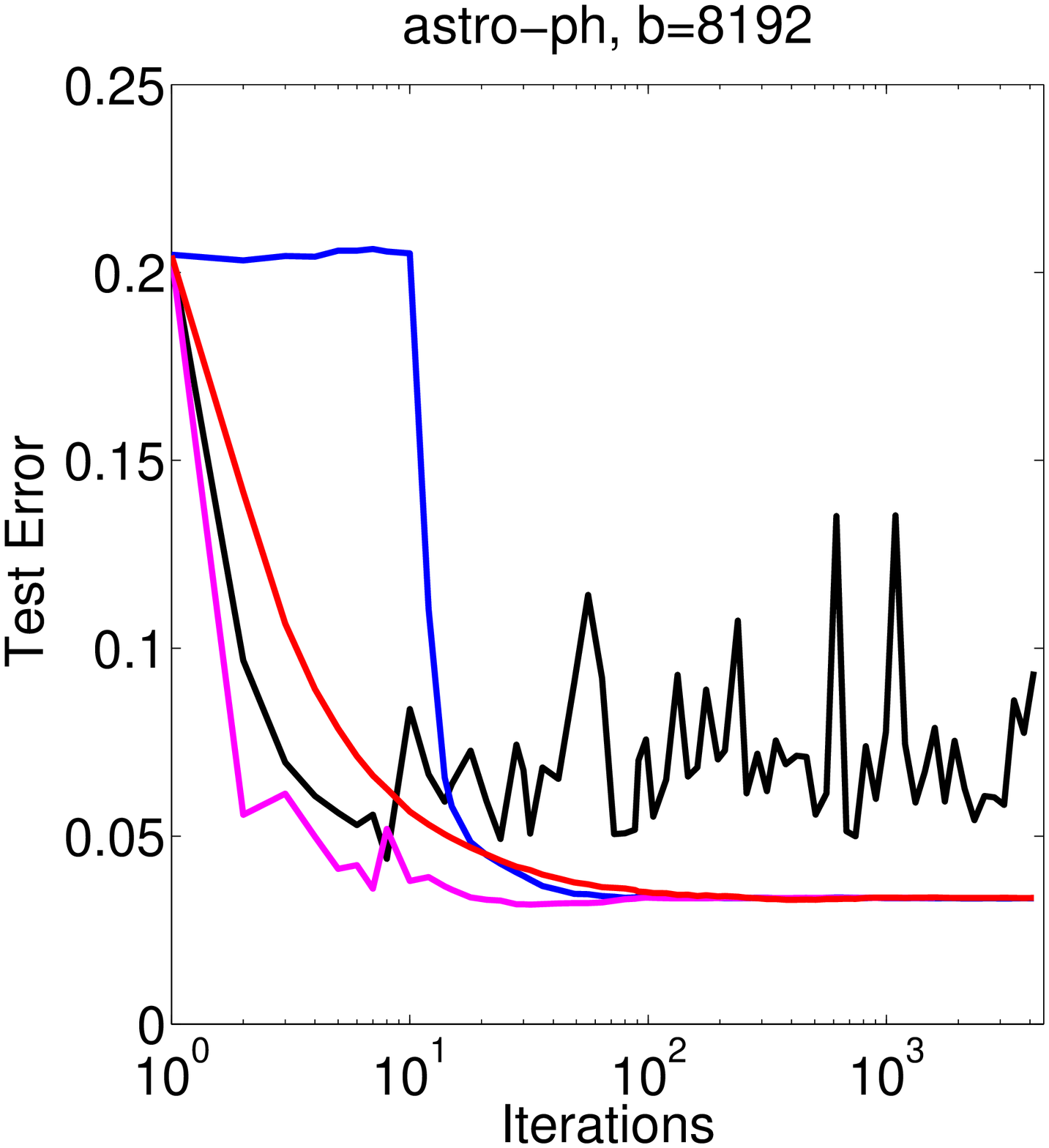}}
\subfigure{\includegraphics[width=1.5in]{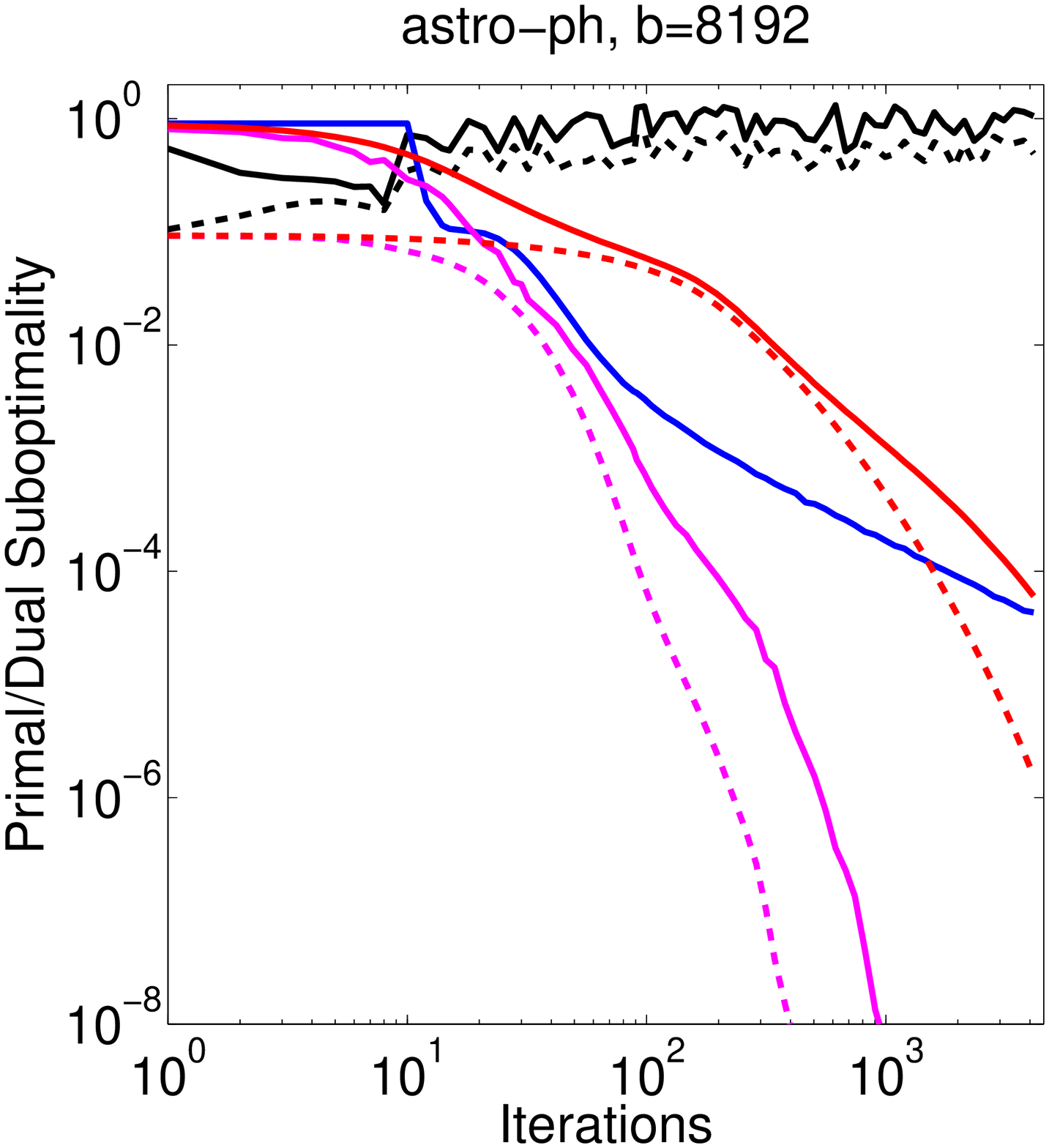}}
\vspace{-0.2in}
\small
 \caption{\small Evolutions of primal (solid) and dual (dashed) sub-optimality and test error  for news20 and astro-ph datasets.
 Instead of tail averaging, in the experiments we used decaying averaging with $\vc{\bar{\w}}{t} = 0.9 \vc{\bar{\w}}{t-1}+ 0.1 \vc{\w}{t}$.}
 \label{fig:failureOfNaiveApproach}
\end{figure*}

Figure~\ref{fig:itertions2epsilon} shows the required number of iterations (corresponding
to the parallel runtime) required for achieving a primal suboptimality
of $0.001$ using Pegasos, naive SDCA, safe SDCA and aggressive SDCA, on four benchmark datasets
detailed in Table~\ref{tab:datasets}, using different mini-batch sizes.  Also shown
(on an independent scale; right axis) is the leading term $\tfrac{\beta_b}{b}$ in our complexity results.  The results
confirm the advantage of SDCA over Pegasos, at least for $b=1$, and
that both Pegasos and SDCA enjoy nearly-linear speedups, at least for
small batch sizes.  Once the mini-batch size is such that $\tfrac{\beta_b}{b}$
starts flattening out (corresponding to $b \approx \tfrac{1}{\sigma^2}$, and so
significant correlations inside each mini-batch), the safe variant of
SDCA follows a similar behavior and does not allow for much
parallelization speedup beyond this point, but at least does not
deteriorate like the naive variant.  Pegasos and the aggressive
variant do continue showing speedups beyond $b \approx \tfrac{1}{\sigma^2}$.
The experiments clearly demonstrate the aggressive modification allows
SDCA to continue enjoying roughly the same empirical speedups as
Pegasos, even for large mini-batch sizes, maintaining an advantage
throughout.  It is interesting to note that the aggressive
variant continues improving even past the point of failure of the
naive variant, thus establishing that it is empirically important to
adjust the step-size to achieve a balance between safety and
progress.

\begin{table}[t!]
\small
\caption{\small Datasets and regularization parameters $\lambda$ used; ``\%'' is
  percent of features which are non-zero. \emph{cov} is the forest covertype dataset of \citet{pegasos}, \emph{astro-ph} consists of abstracts of papers from physics also of \citet{pegasos}, \emph{rcv1} is from the Reuters collection and \emph{news20} is from the 20 news groups both obtained from  libsvm collection \cite{libsvm}.\\}
\label{tab:datasets}
\centering
{
\footnotesize
\begin{tabular}{|c|r|r|r|r|l|}
\hline
 Data & \multicolumn{1}{|c|}{\# train} &
 \multicolumn{1}{|c|}{\# test} &
 \multicolumn{1}{|c|}{\# dim} &
 \multicolumn{1}{|c|}{\%}&
 \multicolumn{1}{|c|}{$\lambda$}
 \\ \hline \hline
 cov & 522,911  & 58,101 & 54 & 22 & $0.000010$

 \\
 rcv1 & 20,242 & 677,399 & 47,236 & 0.16 & $0.000100$
 \\

 astro-ph& 29,882 & 32,487  & 99,757 & 0.08& $0.000050$
 \\

 news20 & 15,020 & 4,976 & 1,355,191 & 0.04 & $0.000125$\\

 \hline
\end{tabular}
}
\end{table}

In Figure~\ref{fig:failureOfNaiveApproach} we demonstrate the evolution of solutions using the
various methods for two specific data sets.  Here we can again see the
relative behaviour of the methods, as well as clearly see the failure
of the naive approach, which past some point causes the objective to
deteriorate and does not converge to the optimal solution.

\section{Conclusion}

\textbf{Contribution.}
Our contribution in this paper is twofold: (i) we identify the
spectral norm of the data, and through it the quantity $\beta_b$, as the important
quantity controlling guarantees for mini-batched/parallelized Pegasos (primal method) and SDCA (dual method).
We provide the first analysis of mini-batched Pagasos, with
the non-smooth hinge-loss, that shows speedups, and we analyze for the
first time mini-batched SDCA with guarantees expressed in terms of the primal problem (hence, our mini-batched SDCA is a primal-dual method); (ii) based on our analysis, we present
novel variants of mini-batched SDCA which are necessary for achieving
speedups similar to those of Pegasos, and thus open the door to
effective mini-batching using the often-empirically-better SDCA.

\textbf{Related work.} Our safe SDCA mini-batching approach is similar to the parallel
coordinate descent methods of \citet{shotgun} and \citet{richtarikBigData},
but we provide an analysis in terms of the primal SVM objective, which
is the more relevant object of interest.  Furthermore,
\citeauthor{shotgun}'s analysis does {\em not} use a step-size and is
thus limited only to small enough mini-batches---if the spectral norm
is unknown and too large a mini-batch is used, their
method might not converge.  \citeauthor{richtarik}'s method does
incorporate a {\em fixed} step-size, similar to our safe variant, but
as we discuss this step-size might be too conservative for achieving
the true potential of mini-batching.

\textbf{Generality.} We chose to focus on Pegasos and SDCA with regularized hinge-loss
minimization, but all our results remain unchanged for any Lipschitz
loss functions.  Furthermore, Lemma~\ref{lemma:nablaL} can also be used
to establish identical speedups for mini-batched SGD optimization of $\min_{\norm{\w}\leq B} \hat{L}(\w)$, as
well as for direct stochastic approximation of the population
objective (generalization error) $\min L(\w)$.  In considering the
population objective, the sample size is essentially infinite, we
sample with replacements (from the population), $\sigma^2$ is a bound
on the second moment of the data distribution, and $\beta_b = 1+(b-1)\sigma^2$.

\textbf{Experiments.} Our experiments confirm the empirical advantages of SDCA over Pegasos,
previously observed without mini-batching.  However, we also point out
that in order to perform mini-batched SDCA effectively, a step-size is
needed, detracting from one of the main advantages of SDCA over
Pegasos.  Furthermore, in the safe variant, this stepsize needs to be
set according to the spectral norm (or bound on the spectral norm),
with too small a setting for $\beta$ (i.e., too large steps) possibly
leading to non-convergence, and too large a setting for $\beta$
yielding reduced speedups.  In contrast, the Pegasos stepsize is {\em
  independent} of the spectral norm, and in a sense Pegasos adapts
implicitly (see, e.g., its behavior compared to aggressive SDCA in the
experiments).  We do provide a more aggressive variant of SDCA, which
does match Pegasos's speedups empirically, but this requires an
explicit heuristic adaptation of the stepsize.

\textbf{Parallel Implementation.} In this paper we analyzed the iteration complexity, and behavior of the
iterates, of mini-batched Pegasos and SDCA.  Unlike ``pure'' ($b$=1)
Pegasos and SDCA, which are not amenable to parallelization, using
mini-batches does provide opportunities for it.  Of course,
actually achieving good parallelization speedups on a specific
architecture in practice requires an efficient parallel, possibly
distributed, implementation of the iterations.  In this regard, we
point out that the core computation required for both Pegasos and SDCA
is that of computing $\sum_{i\in A} g_i(\ve{\w}{\x_i}) \x_i$, where $g$ is
some scalar function.  Parallelizing such computations efficiently in a
distributed environment has been studied by
e.g., \citet{dekel2012optimal,hsu2011parallel};  their methods can be
used here too.  Alternatively, one could also consider asynchronous or
delayed updates \cite{AgarwalDuchi,recht2011hogwild}.

\bibliography{minibatch}

\begin{thebibliography}{16}
\providecommand{\natexlab}[1]{#1}
\providecommand{\url}[1]{\texttt{#1}}
\expandafter\ifx\csname urlstyle\endcsname\relax
  \providecommand{\doi}[1]{doi: #1}\else
  \providecommand{\doi}{doi: \begingroup \urlstyle{rm}\Url}\fi

\bibitem[Agarwal \& Duchi(2011)Agarwal and Duchi]{AgarwalDuchi}
Agarwal, A. and Duchi, J.
\newblock Distributed delayed stochastic optimization.
\newblock In \emph{NIPS}, 2011.

\bibitem[Bradley et~al.(2011)Bradley, Kyrola, Bickson, and Guestrin]{shotgun}
Bradley, J.K., Kyrola, A., Bickson, D., and Guestrin, C.
\newblock Parallel coordinate descent for l1-regularized loss minimization.
\newblock In \emph{ICML}, 2011.

\bibitem[Cotter et~al.(2011)Cotter, Shamir, Srebro, and Sridharan]{CotterEtAl}
Cotter, A., Shamir, O., Srebro, N., and Sridharan, K.
\newblock Better mini-batch algorithms via accelerated gradient methods.
\newblock In \emph{NIPS}, 2011.

\bibitem[Dekel et~al.(2012)Dekel, Gilad-Bachrach, Shamir, and
  Xiao]{dekel2012optimal}
Dekel, O., Gilad-Bachrach, R., Shamir, O., and Xiao, L.
\newblock Optimal distributed online prediction using mini-batches.
\newblock \emph{Journal of Machine Learning Research}, 13:\penalty0 165--202,
  2012.

\bibitem[Hsieh et~al.(2008)Hsieh, Chang, Lin, Keerthi, and Sundarajan]{dcd}
Hsieh, C-J., Chang, K-W., Lin, C-J., Keerthi, S.S., and Sundarajan, S.
\newblock A dual coordinate descent method for large-scale linear svm.
\newblock In \emph{ICML}, 2008.

\bibitem[Hsu et~al.(2011)Hsu, Karampatziakis, Langford, and
  Smola]{hsu2011parallel}
Hsu, D., Karampatziakis, N., Langford, J., and Smola, A.
\newblock Parallel online learning.
\newblock \emph{arXiv:1103.4204}, 2011.

\bibitem[Libsvm()]{libsvm}
Libsvm.
\newblock \emph{Datasets}.
\newblock http://www.csie.ntu.edu.tw/${}^{\sim}$cjlin/
  libsvmtools/datasets/binary.html.

\bibitem[Nesterov(2012)]{NesterovRCDM}
Nesterov, Yu.
\newblock Efficiency of coordinate descent methods on huge-scale optimization
  problems.
\newblock \emph{SIAM J. Optimization}, 22:\penalty0 341--362, 2012.

\bibitem[Niu et~al.(2011)Niu, Recht, Re, and Wright]{recht2011hogwild}
Niu, F., Recht, B., Re, C., and Wright, S.
\newblock Hogwild: A lock-free approach to parallelizing stochastic gradient
  descent.
\newblock In Shawe-Taylor, J., Zemel, R.S., Bartlett, P., Pereira, F.C.N., and
  Weinberger, K.Q. (eds.), \emph{NIPS 24}, pp.\  693--701. 2011.

\bibitem[Rakhlin et~al.(2012)Rakhlin, Shamir, and Sridharan]{MakingSGDOptimal}
Rakhlin, A., Shamir, O., and Sridharan, K.
\newblock Making gradient descent optimal for strongly convex stochastic
  optimization.
\newblock \emph{ArXiv:1109.5647}, 2012.

\bibitem[Richt\'arik \& Tak\'a\v{c}(2012)Richt\'arik and
  Tak\'a\v{c}]{richtarikBigData}
Richt\'arik, P. and Tak\'a\v{c}, M.
\newblock Parallel coordinate descent methods for big data optimization.
\newblock \emph{ArXiv:1212.0873}, 2012.

\bibitem[Richt\'arik \& Tak\'a\v{c}(2013)Richt\'arik and
  Tak\'a\v{c}]{richtarik}
Richt\'arik, P. and Tak\'a\v{c}, M.
\newblock Iteration complexity of randomized block-coordinate descent methods
  for minimizing a composite function.
\newblock \emph{Mathematical Programming}, 2013.
\newblock \doi{10.1007/s10107-012-0614-z}.

\bibitem[Shalev-Shwartz \& Tewari(2011)Shalev-Shwartz and
  Tewari]{TewariShalevShwartzJMLR2011}
Shalev-Shwartz, S. and Tewari, A.
\newblock {Stochastic Methods for l1-regularized Loss Minimization}.
\newblock \emph{JMLR}, 12:\penalty0 1865--1892, 2011.

\bibitem[Shalev-Shwartz \& Zhang(2012)Shalev-Shwartz and
  Zhang]{ShalevShawartzZhang}
Shalev-Shwartz, S. and Zhang, T.
\newblock Stochastic dual coordinate ascent methods for regularized loss
  minimization.
\newblock \emph{ArXiv:1209.1873}, 2012.

\bibitem[Shalev-Shwartz et~al.(2011)Shalev-Shwartz, Singer, Srebro, and
  Cotter]{pegasos}
Shalev-Shwartz, S., Singer, Y., Srebro, N., and Cotter, A.
\newblock Pegasos: Primal estimated sub-gradient solver for svm.
\newblock \emph{Mathematical Programming: Series A and B- Special Issue on
  Optimization and Machine Learning}, pp.\  3--30, 2011.

\bibitem[Zhang(2004)]{zhang}
Zhang, T.
\newblock Solving large scale linear prediction using stochastic gradient
  descent algorithms.
\newblock In \emph{ICML}, 2004.

\end{thebibliography}
\bibliographystyle{icml2013}

\clearpage

\onecolumn

\appendix

\newcommand{\vu}{ \boldsymbol{\chi} }    %

\section{Proof of Theorem \ref{thm:dualityGapForLipFunctions}}

The proof of Theorem~\ref{thm:dualityGapForLipFunctions} follows mostly
along the path of \citet{ShalevShawartzZhang}, crucially using Lemma~\ref{lemma:H}, and with a few other required modifications detailed
below.

We will prove the theorem for a general $L$-Lipschitz loss function
$\ell(\cdot)$.  For consistency with \citeauthor{ShalevShawartzZhang}, we
will also allow example-specific loss functions $\ell_i$, $i=1,2,\dots,n$, and only require each
$\ell_i$ be individually Lipschitz, and thus refer to the primal and
dual problems (expressed slightly differently but equivalently):
\begin{align}\tag{P}\label{eq:P}
 \min_{\w\in\R^d} & \left[\bP(\w) \eqdef \tfrac1n \sum_{i=1}^n \ell_i(\ve{\w}{\x_i})
 +\tfrac\lambda2 \|\w\|^2\right], \\
\tag{D}
 \max_{\alf \in\R^n} & \left[\bD(\alf)\eqdef -\tfrac1n \sum_{i=1}^n \ell_i^*(-\cor{\alf}{i}) -
          \tfrac {\lambda}{2} \left\|\tfrac1{\lambda n} \X^\trans \alf  \right\|_2^2\right],
\end{align}
where $\ell^*_i(u)=\max_z (z u -\ell_i(z))$ is the Fenchel conjugate of
$\ell_i$.  In the above we dropped without loss of generality the
labels $y_i$ since we can always substitute $\x_i \leftarrow y_i
\x_i$.  For the hinge loss $\ell_i(a)=[1-a]_+$ we have
$\ell_i^*(-a) = -a $ for $a \in [0,1]$ and $\ell_i^*(-a)=\infty$
otherwise, thus encoding the box constraints.  Recall also (from \eqref{eq:walpha}) that $\w(\alf)
= \frac{1}{\lambda n} \sum_{i=1}^n \alf_i \x_i$ and so
$\norm{\w(\alf)}^2 = \frac{1}{\lambda^2 n^2} \alf^\trans \X \X^\trans
\alf = \norm{ \frac{1}{\lambda n} \X^\trans \alf}^2$.

The separable approximation $\bH(\vdelta,\alf)$ defined in
\eqref{eq:H} now has the more general form:
\begin{align}\label{eq:H_definition}
 \calH(\vdelta,\alf) &:= -\tfrac1n \sum_{i=1}^n \ell_i^*(-(\cor{\alf}{i}+\cor{\vdelta}{i}))
   - \tfrac\lambda2
  \left(
     \norm{\w(\alf)}^2
     + \beta_b \tfrac1{\lambda n} \sum_{i=1}^n \norm{\x_i}^2 \vdelta_i^2
     +2  \left(\tfrac1{\lambda n}   \vdelta\right)^\trans  \X  \w(\alf) \right)
\end{align}
and all the properties mentioned in Section \ref{sec:SDCA}, including
Lemma \ref{lemma:H}, still hold.

Our goal here is to get a bound on the duality gap, which we will
denote by
\begin{align}
 \calG(\alf)
    &\eqdef \bP(\w(\alf)) - \bD(\alf) \label{eq:dualityGap}
  = \tfrac1n \sum_{i=1}^n \left[ \ell_i(\ve{\w(\alf)}{\x_i})
    +\ell_i^*(-\cor{\alf}{i}) + \cor{\alf}{i} \ve{\w(\alf)}{\x_i} \right].
\end{align}

The analysis now rests on the following lemma, paralleling Lemma 1 of
\citet{ShalevShawartzZhang}, which bounds the expected improvement in the dual
objective after a single iteration in terms of the duality gap:

\begin{lemma}
\label{lem:basicLemma}
For any $t$ and any $s\in[0,1]$ we have
\begin{equation}\label{eq:relationOfDualDecreaseAndDualityGap}
 \Exp_{A_t}[\bD(\vc{\alf}{t+1})]-\bD(\vc{\alf}{t})
\geq
  b \left (
  \tfrac {s  }{n} \calG(\vc{\alf}{t})
  - \left(\tfrac{s}{n}\right)^2
       \tfrac{\beta_b}{ 2  \lambda }
  \vc{G}{t}
\right),
\end{equation}
where
\begin{align}
 \vc{G}{t} &\eqdef \tfrac1n  \label{eq:defOfG}
 \sum_{i=1}^n %
        \|\x_i\|^2
      (\corit{\vu}{i}{t}-\corit{\alf}{i}{t})^2 \leq G,
\end{align}
with $G=4L$ for general $L$-Lipschitz loss, and $G=1$ for the hinge
loss, and $-\corit{\vu}{i}{t}  \in \ell'_i(\ve{\w(\vc{\alf}{t})}{\x_i})$.
\end{lemma}
\begin{proof}
  The situation here is trickier then in the case $b=1$ considered by
  \citeauthor{ShalevShawartzZhang}, and we will first bound the
  right hand side of \eqref{eq:relationOfDualDecreaseAndDualityGap} by
  $\bH-\bD$ and then use the fact that $\vdelta^{(t)}$ is a minimizer of
  $\bH(\cdot,\alf)$:
\begin{align*}
\lefteqn{ -\frac nb \left(\Exp[\bD(\vc{\alf}{t+1})]-\bD(\vc{\alf}{t})\right)
 =-\frac nb \left(\Exp[\bD(\vc{\alf}{t}+\vsubset{\vc{\vdelta}{t}}{A_t})]-\bD(\vc{\alf}{t})\right)
 \overset{\text{(Lemma 3)}}{\leq}
 -\calH(\vc{\vdelta}{t},\vc{\alf}{t})+\bD(\vc{\alf}{t}) }
 \\
 & =
 \frac1n
 \sum_{i=1}^n \left(\ell_i^*(-(\corit{\alf}{i}{t}+\corit{\vdelta}{i}{t}))
  -\ell_i^*(-\corit{\alf}{i}{t}) \right)
+ \frac\lambda2
  \left(
       \beta_b \left\|\frac1{\lambda n} \vc{\vdelta}{t} \right\|^2_\X
     +2  \left(\frac1{\lambda n}   \vc{\vdelta}{t} \right) ^\trans \X \w(\vc{\alf}{t}) \right),
\\
\intertext{where we denote $\norm{{\bf u}}^2_{\X} \eqdef \sum_{i=1}^n {\bf
    u}^2_i \|\cor{\x}{i}\|^2$.  We will now use the optimally of
  $\vc{\vdelta}{t}$ to upper bound the above, noting that if we
  replace $\vc{\vdelta}{t}$ with any quantity, and in particular with
  $s(\vu^{(t)}-\alf^{(t)})$, we can only decrease $\bH(\cdot,\alf^{(t)})$, and
  thus increase the right-hand-side above:}
& \leq
 \tfrac1n
 \sum_{i=1}^n \left[\ell_i^*(-(\corit{\alf}{i}{t}+s(\corit{\vu}{i}{t}-\corit{\alf}{i}{t})))
  -\ell_i^*(-\corit{\alf}{i}{t}) \right]\\
  & \quad\quad\quad
+ \tfrac\lambda2
  \left(
       \beta_b \left\|\tfrac1{\lambda n} s(\vc{\vu}{t}-\vc{\alf}{t}) \right\|^2_\X
     +2  \left(\tfrac1{\lambda n}   s(\vc{\vu}{t}-\vc{\alf}{t}) \right) ^\trans \X \w(\vc{\alf}{t}) \right)\\
\intertext{Now from convexity we have $\ell_i^*(-(\corit{\alf}{i}{t}+s(\corit{\vu}{i}{t}-\corit{\alf}{i}{t})))  \leq s \ell_i^*(\corit{-\vu}{i}{t})
     +(1-s) \ell_i^*(-\corit{\alf}{i}{t})$, and so:}
& \leq
 \tfrac1n
 \sum_{i=1}^n
 \left(
    s \ell_i^*(\corit{-\vu}{i}{t})
    +s\corit{\vu}{i}{t} \ve{\w(\vc{\alf}{t})}{\x_i}
    -s \ell_i^*(-\corit{\alf}{i}{t})
 \right)\\
 & \quad \quad \quad 
 + \tfrac\lambda2
  \left(
       \beta_b \left\|\tfrac1{\lambda n} s(\vu^{(t)}-\vc{\alf}{t}) \right\|^2_\X
     +2  \left(\tfrac1{\lambda n}   s(-\vc{\alf}{t}) \right) ^\trans \X
     \w(\vc{\alf}{t}) \right) \\
\intertext{and from conjugacy we have
  $\ell_i^*(-\corit{\vu}{i}{t})=-\corit{\vu}{i}{t}
  \ve{\w(\vc{\alf}{t})}{\x_i} - \ell_i(\ve{\w(\vc{\alf}{t})}{\x_i})$,
  and so:}
&\leq
 \tfrac sn
 \sum_{i=1}^n
 \left(
     -\corit{\vu}{i}{t} \ve{\w(\vc{\alf}{t})}{\x_i}
     -  \ell_i\left(\ve{\w(\vc{\alf}{t})}{\x_i}\right)
    + \corit{\vu}{i}{t} \ve{\w(\vc{\alf}{t})}{\x_i}
    -  \ell_i^*(-\corit{\alf}{i}{t})
 \right)
 \\&\quad\quad\quad+ \tfrac\lambda2
  \left(
       \beta_b \left\|\tfrac1{\lambda n} s(\vu^{(t)}-\vc{\alf}{t}) \right\|^2_{\X}
            +2  \left(\tfrac1{\lambda n}   s(-\vc{\alf}{t}) \right) ^\trans \X \w(\vc{\alf}{t})
 \right)
\\& \leq
 \tfrac sn
 \sum_{i=1}^n
 \left(
      -  \ell_i\left(\ve{\w(\vc{\alf}{t})}{\x_i}\right)
    -  \ell_i^*(-\corit{\alf}{i}{t})
    - \corit{\alf}{i}{t} \ve{\w(\vc{\alf}{t})}{\x_i}
 \right)
+ \tfrac\lambda2
       \beta_b \left\|\tfrac1{\lambda n} s(\vu^{(t)}-\vc{\alf}{t}) \right\|^2_\X
\\&\overset{\eqref{eq:dualityGap}}{=}
  - s \calG(\vc{\alf}{t})
  + \tfrac1{2 \lambda }
  \left(\tfrac sn\right)^2
  \left(
       \beta_b \left\|(\vu^{(t)}-\vc{\alf}{t}) \right\|^2_\X
      \right).
\end{align*}

Multiplying both sides of the resulting inequality by $\tfrac {-b}n$ we obtain
\eqref{eq:relationOfDualDecreaseAndDualityGap}.  To get the bound on
$G^{(t)}$, recall that $\ell(\cdot)$ is $L$-Lipschitz, hence $-L \leq
\vu_i^{(t)} \leq L$.  Furthermore, $\alf^{(t)}$ is dual feasible, hence
$\ell_i^*(-\alf_i^{(t)}) < \infty$ and so $(-\alf_i^{(t)})$ is a
(sub)derivative of $\ell_i$ and so we also have $-L \leq \alf_i^{(t)}
\leq L$ and for each $i$, and $(\corit{\vu}{i}{t}-\corit{\alf}{i}{t})^2 \leq
4 L$.  For the hinge loss we have $0 \leq \vchi_i^{(t)},\alf_i^{(t)} \leq
1$, and so $(\corit{\vu}{i}{t}-\corit{\alf}{i}{t})^2 \leq 1$.
\end{proof}

We are now ready to prove the theorem.
\begin{proof}[Proof of Theorem~\ref{thm:dualityGapForLipFunctions}]
We will bound the change in the dual sub-optimality $\vc{\epsilon_\bD}{t} \eqdef
  \bD(\alf^*)-\bD(\vc{\alf}{t})$:
\begin{multline}
  \Exp_{A_t}[\vc{\epsilon_\bD}{t+1}]
=
\Exp[\bD(\vc{\alf}{t})-\bD(\vc{\alf}{t+1})+\vc{\epsilon_\bD}{t}]
 \overset{\textrm{(Lemma~\ref{lem:basicLemma})}}{\leq}
 - b \left (
  \tfrac {s  }{n} \calG(\vc{\alf}{t})
  - \left(\tfrac{s}{n}\right)^2
       \tfrac{\beta_b}{ 2  \lambda }
G\right)
+
\vc{\epsilon_\bD}{t}
\\ \overset{\vc{\epsilon_\bD}{t} \leq \calG(\vc{\alf}{t})}{\leq}
- b
  \tfrac {s  }{n} \vc{\epsilon_\bD}{t}
+ b    \left(\frac{s}{n}\right)^2
       \tfrac{\beta_b}{ 2  \lambda }G
       +
\vc{\epsilon_\bD}{t}
 =
(1- b
  \tfrac {s  }{n} )\vc{\epsilon_\bD}{t}
+ b    \left(\frac{s}{n}\right)^2
       \tfrac{\beta_b}{ 2  \lambda }G.
         \label{eq:expectedBound}
\end{multline}
Unrolling this recurrence, we have:
\begin{align}\nonumber
  \Exp[\vc{\epsilon_\bD}{t}]
&\leq
(1- b
  \tfrac {s  }{n} )^t
  \vc{\epsilon_\bD}{0}
+
b    \left(\tfrac{s}{n}\right)^2
       \frac{\beta_b}{ 2  \lambda }G
\sum_{i=0}^{t-1} (1- b
  \tfrac {s  }{n} )^i
\leq
(1- b
  \tfrac {s  }{n} )^t
  \vc{\epsilon_\bD}{0}
+
    \left(\tfrac{s}{n}\right)
       \tfrac{\beta_b G}{ 2  \lambda }.
\end{align}
Setting  $s=1$ and
\begin{equation}
  \label{eq:t0}
t_0:= [\lceil \tfrac nb \log(2\lambda n
\vc{\epsilon_\bD}{0}/ (G\beta_b)) \rceil]_+
\end{equation}
yields:
\begin{align}\label{eq:induction_step1}
  \Exp[\vc{\epsilon_\bD}{t_0}]
 &\leq
(1-
  \tfrac {b  }{n} )^{t_0}
  \vc{\epsilon_\bD}{0}
+
     \frac{s}{n}
       \frac{\beta_b G}{ 2  \lambda }
\leq
\frac{G\beta_b}{2\lambda n \epsilon_\bD} \epsilon_\bD
+
    \frac{1}{n}
       \frac{\beta_b G}{ 2  \lambda }
=
       \frac{\beta_b G}{ \lambda n   }.
\end{align}
Following the proof of \citeauthor{ShalevShawartzZhang} we will now show by
induction that
\begin{equation}\label{eq:expectationOfDualFeasibility}
\forall t\geq t_0 :  \Exp[\vc{\epsilon_\bD}{t}] \leq \frac{2 \beta_b G}{\lambda (2n+b(t-t_0))}.
\end{equation}
Clearly, \eqref{eq:induction_step1} implies that \eqref{eq:expectationOfDualFeasibility} holds for $t=t_0$.
Now, if it holds for some $t\geq t_0$, we show that it also holds for
$t+1$. Using $s=\frac{2n}{2n+b(t-t_0)}$ in \eqref{eq:expectedBound} we have:
\begin{align}
\Exp[\epsilon_\bD^{(t+1)}]
&\overset{\eqref{eq:expectedBound}}{\leq}
(1- b
  \tfrac {s  }{n} )\Exp[\vc{\epsilon_\bD}{t}]
+ b    \left(\frac{s}{n}\right)^2
       \frac{\beta_b}{ 2  \lambda }G
\overset{\eqref{eq:expectationOfDualFeasibility}}{\leq}
(1- b
  \tfrac {s  }{n} ) \frac{2 \beta_b G}{\lambda (2n+b(t-t_0))}
+ b    \left(\frac{s}{n}\right)^2
       \frac{\beta_b}{ 2  \lambda }G
\nonumber%
\\
&=
(1- b
  \frac{2}{2n+b(t-t_0)} ) \frac{2 \beta_b G}{\lambda (2n+b(t-t_0))}
   + b    \left(\frac{2}{2n+b(t-t_0)}\right)^2
       \frac{\beta_b}{ 2  \lambda }G
\nonumber%
\\
&=
\tfrac{2G\beta_b}{\lambda (2n+b(t-t_0)+b)}
\tfrac{(2n+b(t-t_0)+b)(2n+b(t-t_0)-b)}{(2n+b(t-t_0))^2}
\leq
\label{eq:fa2f2ffvgafda}
\frac{2G\beta_b}{\lambda (2n+b(t-t_0)+b)},
\end{align}
where in the last inequality we used the arithmetic-geometric mean
inequality.  This establishes \eqref{eq:expectationOfDualFeasibility}.

Now, for the average $\bar{\alf}$ defined in
\eqref{eq:averageDefinition} we have:
\begin{align}
\Exp[\calG(\bar{\alf})] &=
 \Exp\left[\calG\left(\sum_{t=T_0}^{T-1} \tfrac1{T-T_0} \vc{\alf}{t}\right)\right]
 \leq
  \tfrac1{T-T_0} \Exp\left[\sum_{t=T_0}^{T-1} \calG\left( \vc{\alf}{t}\right)\right]
\nonumber\\
\intertext{Applying Lemma \ref{lem:basicLemma} with
  $s=\frac{n}{b(T-T_0)}$:}
&\leq
  \frac {nb(T-T_0)}{n b} \frac{1}{T-T_0}
    \left( \Exp[\bD(\vc{\alf}{T})] -\Exp[\bD(\vc{\alf}{T_0})]
\right)
+
       \frac{G \beta_b n}{ 2 n b (T-T_0) \lambda }
\nonumber\\&
\leq
    \left( \bD(\alf^*) -\Exp[\bD(\vc{\alf}{T_0})]
\right)
+
       \frac{G \beta_b}{ 2 b(T-T_0) \lambda } \nonumber\\
&\overset{\eqref{eq:expectationOfDualFeasibility}}{\leq}
      \left( \frac{2 \beta_b G}{\lambda (2n+b(T_0-t_0))}
\right)
+       \frac{G \beta_b}{ 2 b(T-T_0)  \lambda }
\nonumber\\
\intertext{and if $T\geq \lceil\frac nb\rceil+T_0$ and $T_0\geq t_0$:}
&\leq
\frac{\beta_b G}{b \lambda}
\left(
    \frac{2 }{2\frac nb+ (T_0-t_0)}
+ \frac{1}{2(T-T_0)}
\right). \label{eq:rhsofbound}
\end{align}
Now, we can ensure the above is at most $\epsilon$ if we require:
\begin{align}
T_0 - t_0 &\geq \frac{\beta_b}{b} \left( \frac{4 G}{\lambda \epsilon} - 2\frac{n}{\beta_b}\right)\label{eq:T0}
\\
T-T_0 &\geq \frac{\beta_b}{b}
\frac{ G}{\lambda \epsilon_\calG}. \label{eq:T}
\end{align}

Combining the requirements \eqref{eq:t0}, \eqref{eq:T0} and
\eqref{eq:T} with $T\geq \lceil\frac nb\rceil+T_0$ and $T_0\geq t_0$,
and recalling that for the hinge loss $G=1$ and with $\alf^{(0)}={\bf 0}$ we
have $\epsilon^{(0)}_\bD = \bD(\alf^*)-\bD({\bf 0}) \leq 1-0=1$ gives the
requirements in Theorem~\ref{thm:dualityGapForLipFunctions}.
\end{proof}

\end{document}